\documentclass[twoside]{article}

\usepackage[accepted]{aistats2021}
\usepackage{natbib}
\RequirePackage[loading]{tracefnt}

\PassOptionsToPackage{numbers}{natbib}

\usepackage[utf8]{inputenc} 
\usepackage[T1]{fontenc}    
\usepackage{hyperref}      
\usepackage{url}            
\usepackage{booktabs}       
\usepackage{amsfonts}       
\usepackage{nicefrac}       
\usepackage{microtype}      

\usepackage{microtype}
\usepackage{graphicx}
\usepackage{booktabs} 

\usepackage{algorithm, algorithmic}
\usepackage{amsthm,amsmath,amssymb}
\usepackage{dsfont}
\usepackage{xspace}
\usepackage{tikz}
\usepackage{subcaption}
\usetikzlibrary{positioning}
\usepackage{mathtools}
\usepackage{thmtools}
\usepackage{thm-restate}
\usepackage{multicol}

\newtheorem{theorem}{Theorem}
\newtheorem{lemma}{Lemma}
\newtheorem{corollary}{Corollary}

\DeclareMathOperator*{\argmax}{arg\,max}

\newcommand{\defeq}{\vcentcolon=}
\newcommand{\rd}{R_{\Delta}}
\newcommand{\rmax}{r_{max}}
\newcommand{\ucrl}{{UCRL2}\xspace}
\newcommand{\algo}{{Heavy-UCRL2}\xspace}
\newcommand{\algoq}{{Heavy-Q-Learning}\xspace}

\newcommand{\ait}{\alpha^i_t}
\newcommand{\oeps}{\frac{1}{1+\epsilon}}
\newcommand{\eeps}{\frac{\epsilon}{1+\epsilon}}
\newcommand{\indl}{\mathds{1}_{|r_h(s,a)| \leq B_t}}

\newcommand{\nsv}{\mathbb{V}_hV_{h+1}}
\newcommand{\pvo}{\mathbb{P}V^*_{h+1}}
\newcommand{\skh}{\sum_{k=1}^K\sum_{h=1}^H}

\begin{document}

\twocolumn[

\aistatstitle{No-Regret Reinforcement Learning with Heavy-Tailed Rewards}

\aistatsauthor{ Vincent Zhuang \And Yanan Sui}

\aistatsaddress{Tsinghua University\\ vczhuang@gmail.com \And Tsinghua University\\ ysui@tsinghua.edu.cn }] 

\begin{abstract}
Reinforcement learning algorithms typically assume rewards to be sampled from light-tailed distributions, such as Gaussian or bounded. However, a wide variety of real-world systems generate rewards that follow heavy-tailed distributions. We consider such scenarios in the setting of undiscounted reinforcement learning. By constructing a lower bound, we show that the difficulty of learning heavy-tailed rewards asymptotically dominates the difficulty of learning transition probabilities. Leveraging techniques from robust mean estimation, we propose \algo and \algoq, and show that they achieve near-optimal regret bounds in this setting. Our algorithms also naturally generalize to deep reinforcement learning applications; we instantiate Heavy-DQN as an example of this. We demonstrate that all of our algorithms outperform baselines on both synthetic MDPs and standard RL benchmarks.
\end{abstract}
\section{Introduction}
\label{sec:introduction}

A wide variety of real-world online decision-making systems generate rewards according to heavy-tailed distributions, such as financial instrument prices, networking routing times, scheduling, hydrology, image and audio noise, localization errors, and others \citep{georgiou1999alpha, hamza2001image, huang2017new, ruotsalainen2018error}. Since many potential applications of reinforcement learning (RL) are necessarily dependent on such measurements, there is a clear need for RL algorithms that can handle heavy-tailed reward distributions in a provably efficient manner.

Unfortunately, the prevailing assumption in existing reinforcement learning algorithms is that the reward noise is bounded or light-tailed (e.g. sub-Gaussian). Given this, the difficulty of learning the transition probabilities of the Markov decision process (MDP) dominates the difficulty of learning the rewards; in fact, most analyses assume the rewards to be deterministic without loss of generality \citep{dann2015sample}. As we will demonstrate in this paper, this no longer holds in the heavy-tailed regime.

Although there are no existing RL algorithms that are designed to handle heavy-tailed rewards, several works have studied heavy-tailed feedback in comparatively less general online learning framework. \citet{bubeck2013bandits} are the first to consider the finite-armed bandit problem in which the reward distributions have only finite $(1+\epsilon)$-th moments for $\epsilon \in (0, 1]$. They propose the Robust UCB algorithm, which leverages robust mean estimators such as truncated mean and median of means that have tight concentration properties. Under this setting, several extensions of the vanilla bandit problem have been studied: linear bandits \citep{medina2016no, shao2018almost}, pure exploration (best-arm identification) \citep{yu2018pure}, Lipschitz bandits \citep{lu2019optimal}, and Bayesian optimization \citep{chowdhury2019bayesian}. Thompson sampling has also been analyzed under specific families of heavy-tailed distributions: symmetric $\alpha$-stable distributions \citep{dubey2019thompson} and one-parameter exponential family distributions \citep{korda2013thompson}.


Many wide-ranging definitions of robustness have been considered in reinforcement learning, including worst-case performance of the learned policy \citep{smirnova2019distributionally} and adversarial perturbations of the agent \citep{morimoto2005robust, pinto2017robust}. The closest works to our setting are those that consider robustness with respect to the reward distributions, e.g. in perturbed or corrupted rewards \citep{everitt2017reinforcement, wang2018reinforcement}, or adversarial attacks on the rewards \citep{Weng2020Toward}. Relatively few works have proposed empirical methods for dealing with noisy reward distributions \citep{moreno2006noisy, romoff2018reward}. In general, there is a lack of theoretical and empirical studies on robustness with respect to noisy reward distributions.

The median-of-means estimator is a commonly-used strategy for performing robust mean estimation in heavy-tailed bandit algorithms. In an orthogonal line of work, \citet{pazis2016improving} use median-of-means in order to achieve PAC exploration bounds that scale with the variance of the Bellman operator. However, to the best of our knowledge, no prior work has considered our setting of heavy-tailed rewards in the MDP setting.

\textbf{Our Contribution.} We demonstrate that robust mean estimation techniques can be broadly applied to reinforcement learning algorithms (specifically confidence-based methods) in order to provably handle the heavy-tailed reward setting. We instantiate and prove a variety of theoretical guarantees for two representative algorithms, \algo and \algoq. We show via a lower bound construction that in the unbounded heavy-tailed setting, learning the rewards is non-trivial and leads to a regret term that dominates that due to the transition probabilities. We also propose Heavy-DQN to show that our ideas can be extended to the deep reinforcement learning setting. Finally, we empirically demonstrate that our algorithms significantly outperform standard RL algorithms that are not designed to handle heavy-tailed rewards.

\section{Problem Statement}
\label{sec:problem}

We consider a Markov Decision Process (MDP), defined as a tuple $M = (\mathcal{S}, \mathcal{A}, p, R)$ consisting of the state space $\mathcal{S}$, action space $\mathcal{A}$, transition kernel $p: \mathcal{S} \times \mathcal{A} \times \mathcal{S} \rightarrow [0, 1]$, and stochastic reward function is $r: \mathcal{S} \times \mathcal{A} \rightarrow \mathbb{R}$. An agent begins at some initial state $s_0 \in \mathcal{S}$ (possibly drawn from some distribution), and at each timestep $t$, selects an action $a_t$ and transitions to state $s_{t+1}$ as dictated by $p$. The agent also receives a reward $r_t = r(s_{t}, a_{t})$. A policy $\pi: S \rightarrow A$ gives an agent's action at any state, and the value function $V^{\pi}(s)$ denotes the expected return of policy $\pi$ executed from state $s$. We also define the $Q$-value (also state-action value) function $Q^\pi(s,a)$ as the expected return of executing policy $\pi$ starting with action $a$ at state $s$.

Typically, $r$ is a stochastic function mapping $\mathcal{S} \times \mathcal{A}$ to a bounded interval, e.g. $[0,1]$. In this work, we consider MDPs that have reward functions that follow heavy-tailed distributions, and consequently have unbounded range. Following bandit literature \citep{bubeck2013bandits}, we only assume that the reward distribution has finite $(1+\epsilon)$-th moments for some $\epsilon\in (0, 1]$.  Then the mean reward function is also necessarily finite, which we define as a nonstochastic function $\bar{r} : \mathcal{S} \times \mathcal{A} \rightarrow [r_{min}, r_{max}]$. Note that the reward for each state-action pair $(s,a)$ may be drawn from a different distribution, each with some finite $(1+\epsilon_{s,a})$-th moment. In this case, we can take $\epsilon$ to be that of the heaviest distribution, i.e. $\epsilon = \min_{s,a}\epsilon_{s,a}$. The goal of the agent is to maximize the standard undiscounted return: $\sum_t r_t$. An equivalent object is to minimize its \emph{regret}, defined as $\Delta(s_0, T) \defeq \sum_{i=1}^T \rho^* - \bar{r}(s_t, a_t)$ where $\rho^* \defeq \max_{\pi}\lim_{T\rightarrow\infty}\mathbb{E}\left[\sum_{t=1}^T \bar{r}(s_t, a_t)\right]$.

In order to make the learning problem tractable, we consider two restricted classes of MDPs:

\begin{itemize}
    \item Communicating MDPs, i.e. MDPs $M$ with finite diameter $D$ defined as $D(M) \defeq \max_{s\neq s'}\min_{\pi} \mathbb{E}\left[T(s'|M, \pi, s)\right]$ where $T(s'|M, \pi, s)$ is the random variable corresponding to the first time step in which state $s'$ is reached under policy $\pi$ starting from state $s$.
    \item Finite-horizon episodic MDPs, in which the agent proceeds in episodes of strictly length $H$, and the transition probability kernel can change at each step (i.e. we have $p_h$). The algorithm then runs for time horizon $T=KH$ where $K$ is the number of episodes. In this setting, the diameter is given by $H$.
\end{itemize}
\section{Algorithms}
\label{sec:algorithm}

In this section, we present two algorithms, \algo and \algoq, for reinforcement learning with heavy-tailed rewards. Our algorithms are based on the \ucrl and $Q$-learning algorithms respectively, which are representative model-based and model-free algorithms \citep{jaksch2010near,jin2018q}). \algo and \algoq leverage ideas from the robust mean estimation literature, which provides estimators with Hoeffding-style concentration bounds under only the assumption of bounded $(1+\epsilon)$-th moments, for $\epsilon \in (0, 1]$. Confidence-set based algorithms such as UCRL2 and UCFH are immediately amenable to this technique; we also show that the basic idea of truncation also applies to certain bonus-based model-free algorithms such as Q-learning. We also propose a Heavy-DQN, a direct extension of \algoq to the deep RL setting.

We first briefly review \ucrl and Q-learning. \ucrl is a model-based RL algorithm that implements the broad principle of optimism in the face of uncertainty. At the beginning of each episode, \ucrl constructs confidence sets for all transition probabilities and rewards (eqs. \ref{eq:ucrl_rew} and \ref{eq:p_interval}), which comprise an (infinite) set of statistically plausible MDPs $\mathcal{M}$. \ucrl then compute an optimistic MDP $\widetilde{M} \in \mathcal{M}$ via extended value iteration, which maximizes the optimal average reward. It then executes the optimal policy on $\widetilde{M}$, $\widetilde{\pi}$, for the entirety of the next episode. 

\begin{equation}
    \label{eq:ucrl_rew}
     \|(\widetilde{r} - \hat{r}_k)(s,a)\| \leq \sqrt{\frac{7\log\left(\frac{2SAt_k}{\delta}\right)}{2\hat{N}_k(s,a)}}
\end{equation}

\begin{equation}
\label{eq:p_interval}
    \left\| \widetilde{p}(\cdot | s, a) - \hat{p}_k(\cdot | s,a) \right\|_1 \leq \sqrt{\frac{14S\log\left(\frac{2At_k}{\delta}\right)}{\hat{N}_k(s,a)}}
\end{equation}

\paragraph{Notation.} We use the subscript $k$ to denote the $k$-th episode. Hence, $t_k$ denote the total number of iterations at the start of episode $k$, $v_k(s,a)$ denotes the state-action visitation count within episode $k$, and $N_k(s,a)$ denotes the state-action visitation up to episode $k$. We use $\hat{r}(s,a)$ and $\bar{r}(s,a)$ to denote the algorithm's estimated mean for the reward and the actual mean reward respectively.


\paragraph{Q-learning.} Q-learning directly learns the $Q$-values by repeatedly applying the following one-step Bellman update:
\begin{equation}
\label{eq:q_update}
    Q_h(s,a) = (1-\alpha_t)Q_h(s,a) + \alpha_t[r_h(s,a) + V_{h+1}(s') + b_t]
\end{equation}
where $Q_h$ and $V_h$ are the state and state-action value functions at step $h\leq H$, $t$ is the total number of times the transition $(s,a)$ has been taken thus far, $\alpha_t \in (0, 1]$ is the learning rate, and $b_t$ is an exploration bonus. \citet{jin2018q} show that Q-learning has near-optimal regret bounds in the episodic MDP setting if the $Q$-values are optimistically initialized to $H$ and $\alpha_t$ and $b_t$ are chosen appropriately. Specifically, they use learning rate $\alpha_t \defeq \frac{H+1}{H+t}$, and propose two possible exploration bonuses. The first, UCB-Hoeffding, follows from an application of the Azuma-Hoeffding inequality: $c\sqrt{\frac{H^3\iota}{t}}$ where $\iota \defeq \log(SAT/\delta)$. The second, UCB-Bernstein leverages the Freedman inequality, which allows for a tighter confidence interval based on the empirical estimate of the variance of the next-state value. 

\subsection{Heavy-tailed Mean Estimation}


\paragraph{Robust mean estimators.} We define a robust mean estimator to be an estimator $\hat{\mu}$ that satisfies the following property. Let $X_1, \ldots, X_n$ be i.i.d. random variables with finite mean $\mu$, $\epsilon \in (0, 1]$ and $v$ a bound on the centered $(1+\epsilon)$-th moment of $X$, i.e. $\mathbb{E}[(X-\mu)^{1+\epsilon}] \leq v$ (we will also denote an upper bound on the uncentered $(1+\epsilon)$-th moment of $X$ by $u\geq\mathbb{E}[X^{1+\epsilon}]$). Then, for all $\delta \in (0, 1)$ there exists a constant $c$ such that with probability at least $1-\delta$,
\begin{equation}
\label{eq:robust_conc}
  |\mu - \hat{\mu}| \leq v^{\oeps} \left(\frac{c\log(1/\delta)}{n}\right)^{\eeps}  
\end{equation}

We note that if $\epsilon=1$, i.e. $X$ has finite variance, then $v$ is simply the variance factor, and we recover a Hoeffding-style bound on the deviation of $\hat{\mu}$. Further details can be found in \cite{bubeck2013bandits}. We now review two major robust mean estimators that satisfy eq. \ref{eq:robust_conc}:
\begin{itemize}
    \item \textbf{Truncated empirical mean.} This estimator truncates samples with too high magnitude to zero:
      \begin{equation}
          \label{eq:trunc}
          \hat{\mu}_T = \frac{1}{n} \sum_{t=1}^n X_t \mathds{1}_{\left\{|X_t| \leq \left(\frac{ut}{\log(\delta^{-1})}\right)^{\frac{1}{1+\epsilon}}\right\}}
      \end{equation}
    where $c=4^{(1+\epsilon)/\epsilon}$.
    \item \textbf{Median-of-means.} This simple estimator simply splits the $n$ samples into $k$ blocks and takes the median of the mean in each block:
      \begin{equation}
        \label{eq:mom}
        \hat{\mu}_M = \text{median}\left(\left\{\frac{1}{N}\sum_{t = (i-1) N + 1}^{iN} X_t \right\}_{i \leq N}\right)
      \end{equation}
    where $N=\lceil \frac{n}{k} \rceil$, $k=\lfloor8\log(e^{1/8}/\delta) \land n/2\rfloor$ and $c=32\sqrt[\epsilon]{12}$.
\end{itemize}

\paragraph{Failure of empirical mean/\ucrl for heavy-tailed rewards.} We first demonstrate that naively applying vanilla \ucrl (i.e. with the empirical mean estimator for the rewards) to the heavy-tailed setting results in either a statistically-incorrect or vacuous algorithm. As shown by \citet{bubeck2013bandits}, the tightest possible confidence interval for the empirical mean scales as $O\left(\left(\frac{1}{\delta t^{\epsilon}}\right)^\oeps\right)$. Since the analysis of \ucrl requires that $\delta$ inversely scales polynomially in $t$, the appropriate confidence interval for the rewards scales polynomially in $t$, which clearly dominates the (constant) reward means. Thus, attempting to construct a version of empirical-mean \ucrl with guarantees in the heavy-tailed setting results in an algorithm that eventually devolves into random exploration.

\subsection{\algo}

\begin{algorithm}[t]
\caption{\algo}\label{alg:robustucrl2}
\begin{algorithmic}[1]
\STATE input: confidence parameter $\delta \in (0, 1)$, $\mathcal{S}$, $\mathcal{A}$, $\hat{r}$ a robust mean estimator
\STATE $t\gets 1$, initial state $s_1$
\FOR {episodes $k = 1, 2, \ldots $} 
        \STATE $t_k \gets t$, set $N_k(s,a)$
        \STATE For all $(s,a) \in S \times A$ initialize $v_k(s,a)$ to 0
        \STATE Set $\mathcal{M}_k$ to be the set of all MDPs with states $S$ and actions $A$ with transitions close to $\hat{p}_k(\cdot|s,a)$ and rewards close to $\hat{r}(s,a)$ according to eqs. \ref{eq:p_interval} and \ref{eq:r_interval}
        \STATE Compute policy $\widetilde{\pi}_k$ using extended value iteration on optimistic MDP $\tilde{M}_k = (\mathcal{S}, \mathcal{A}, \tilde{P}_k$, $\tilde{r}_k)$.
        \STATE Execute $\widetilde{\pi}_k$ for episode $k$ until $v_l(s_t, \widetilde{\pi}_k(s_t)) = \max\{1,  N_k(s_t, \widetilde{\pi}_k(s_t)) \}$
\ENDFOR
\end{algorithmic}
\end{algorithm}

We propose \algo, which utilizes a robust estimator for $r(s,a)$ instead of the empirical mean. This allows us to use the following confidence intervals for the rewards:
\begin{equation}
\label{eq:r_interval}
    \|\widetilde{r}(s,a) - \hat{r}_k(s,a)\| \leq v^{\oeps}\left(\frac{7c\log(2SAt_k / \delta)}{\max\{1, N_k(s,a)\}}\right)^{\eeps},
\end{equation}
where $c$ is given as in eq. \ref{eq:heavy_cond}. The confidence intervals for the transition probabilities remain unchanged from eq. \ref{eq:p_interval}. The pseudocode for \algo is given in Algorithm \ref{alg:robustucrl2}.

\paragraph{Comparison with \ucrl.} Here, we briefly emphasize the value of having tight statistically-valid confidence intervals. For \ucrl with the empirical mean estimator for the reward, there are two possibilities for the confidence interval. One can either use the tight confidence interval for the bounded/sub-Gaussian noise case, or a statistically-valid yet loose bound from \cite{bubeck2013bandits}. In the latter case, the $1/\delta$ factor is not hidden inside a logarithm and hence causes the confidence intervals to blow up polynomially in $t$. In the former case, we empirically demonstrate in section \ref{sec:experiment} that under-estimating the confidence interval can degrade the performance of the algorithm. In general, it is clear that using statistically incorrect confidence intervals can cause the algorithm to fail, and having tight confidence intervals is always better, as to avoid unnecessary exploration. 

\subsection{\algoq}

\begin{algorithm}[t]
\caption{\algoq}\label{alg:robustqlearning}
\begin{algorithmic}[1]
\STATE input: confidence parameter $\delta \in (0, 1)$, $\mathcal{S}$, $\mathcal{A}$, bonus $b_t$ either UCB-Hoeffding or UCB-Bernstein, $r_{max}$ the maximum possible mean reward
\STATE initialize $Q_h(s,a) \leftarrow Hr_{max}$ and $N_h(s,a)\leftarrow 0$ for all $(s,a,h) \in \mathcal{S} \times \mathcal{A} \times [H]$
\FOR {episode $k = 1, \ldots, K$}
    \FOR {episodes $h = 1, \ldots, H$} 
        \STATE Take action $a_h\leftarrow \argmax_{a'} Q_h(s_h, a')$
        \STATE $t = N_h(s_h, a_h) \leftarrow N_h(s_h, a_h) + 1$; $b_t' \leftarrow b_t + c_2Hu^{\oeps}\left(\frac{\iota}{t}\right)^{\eeps}$
        \STATE $Q_h(s_h, a_h) \leftarrow (1-\alpha_t)Q_h(s_h, a_h) + \alpha_t[r_h(s_h, a_h)\indl + V_{h+1} + b_t']$
    \ENDFOR
\ENDFOR
\end{algorithmic}
\end{algorithm}

Our \algoq algorithm, detailed in Algorithm \ref{alg:robustqlearning}, extends the basic Q-learning algorithm in a simple way: instead of using the reward directly in the Bellman update, we truncate it to zero if the reward has magnitude too large, and augment the exploration bonus with a term designed to deal with the heavy-tailedness of the rewards. Hence, the \algoq update can be written as follows:
\begin{equation}
\label{eq:hq_update}
    Q_h(s,a) = (1-\alpha_t)Q_h(s,a) + \alpha_t[r_h(s,a)' + V_{h+1}(s') + b_t']
\end{equation}

where 
\begin{equation}
    r_h(s,a)' \defeq r_h(x,a)\indl
\end{equation}
and
\begin{equation}
b_t'\defeq b_t + 8Hu^{\oeps}\left(\frac{\log(2SAT/\delta)}{t}\right)^{\eeps}.
\end{equation}
We define $B_t\defeq \left(\frac{ut}{\log(2SAT/\delta)}\right)^{\oeps}$ as in the truncated mean estimator (eq. \ref{eq:trunc}); we use $\delta'=\frac{\delta}{2SAT}$ in order to facilitate a union bound over $S, A, T$. The base exploration bonus $b_t$ is either of the original UCB-Hoeffding or UCB-Bernstein bonuses proposed in \citet{jin2018q} scaled by $r_{max}$. Our algorithm is compatible with both bonuses since they are primarily designed to deal with the uncertainty in transition probabilities.

As we will show in the following section, this simple modification allows \algoq to be robust to heavy-tailed rewards while also preserving all of the advantages of standard Q-learning over UCRL2, such as its smaller time and space complexities. We note that for this reason only the truncated mean estimator is amenable to Q-learning; the median-of-means estimator would require recomputing the entire history at each step. Finally, an additional modification we make in \algoq is to optimistically initialize the $Q$-function to uniformly be $Hr_{max}$. This is one disadvantage of \algoq over \algo (and in general Q-Learning vs UCRL2): the former requires explicit a priori knowledge of $r_{max}$, whereas the latter does not. 
\subsection{Heavy-DQN}
Recently, there has been much focus on deep reinforcement learning algorithms and applications. To demonstrate that the general ideas behind robust mean estimation can be extended to this setting, we propose Heavy-DQN as a combination of \algoq and DQN for discrete-action deep reinforcement learning \citep{mnih2013playing}. To perform adaptive truncation in non-tabular environments, we use methods from the count-based exploration literature, e.g. the SimHash function $\phi$ \citep{tang2017exploration}. Since we want to estimate the count per state-action pair, we consider the SimHash counts per action. Then, our reward truncation threshold is given by $(C \cdot n(\phi(s, a)))^{\oeps}$, where $C$ is an environment-dependent hyperparameter. Although there is recent work on extending optimistic initialization and UCB exploration from Q-learning to the deep RL setting, in our experiments it was sufficient to use standard Q-learning initialization and linearly-decaying epsilon-greedy exploration \citep{rashid2020optimistic}. 
\paragraph{Connections to reward clipping.} Deep RL algorithm implementations commonly clip rewards to some fixed range $[-x, x]$ or even as $\text{sign}(r)$ \citep{mnih2013playing}. The reward truncation in Heavy-DQN can be viewed as an adaptive version of this kind of fixed clipping. We note the differing original intents of these mechanisms: the main purpose of reward clipping is to stabilize the training dynamics of the neural networks, whereas our method is designed to ensure theoretically-tight reward estimation in the heavy-tailed setting for each state-action pair. Whether there are any unifying connections between these two purposes is an open area for future work.

\section{Theoretical Results}
\label{sec:theoretical}

In this section, we present a variety of theoretical results for \algo and \algoq as well as a general lower bound for the heavy-tailed RL setting. In particular, for \algo and \algoq with both bonuses, we prove standard minimax upper bounds on the expected regret. In addition, for \algo we prove analogous results for every result in \citet{jaksch2010near}: a regret bound dependent on the gap between the best and second-best policies and a regret bound for a MDP that is allowed to change a fixed number of times. As noted in \citep{jin2018q}, our regret bounds also naturally translate to PAC bounds.

In all the following, we denote $S\defeq |\mathcal{S}|$ and $A\defeq |\mathcal{A}|$, and use $\iota$ to hide the logarithmic factor $\log\left(\frac{2SAT}{\delta}\right)$. For the sake of conciseness, we primarily state the theorems in this section, and defer the detailed proofs to the appendix.

\subsection{\algo results}

\paragraph{Minimax regret bound.} We show that \algo enjoys an upper bound on worst-case regret of $\tilde{O}(DS\sqrt{AT} + (SAT)^{\oeps})$ for sufficiently large $T$. Note that if the variance is bounded, we recover an identical bound as in the bounded rewards case: $\tilde{O}(DS\sqrt{AT})$. This is formalized in the following theorem:
\begin{theorem}
\label{thm:minimax}
Let $\rd \defeq r_{max} - r_{min}$. With probability at least $1-\delta$, the regret of \algo is bounded by 
\begin{equation}
\begin{split}
&20\rd DS\sqrt{AT\log{\left(\frac{T}{\delta}\right)}} +\\ &(2C_\epsilon + 1) v^{\oeps} \left(7c\iota\right)^{\eeps} \left(SAT\right)^{\oeps}
\end{split}
\end{equation}
\end{theorem}

The first term corresponds to the regret incurred due to the uncertainty in transition probabilities, and the second term to the regret due to the heavy-tailed nature of the rewards. In the standard bounded reward setting, the first term is asymptotically dominant. However, since $D,S,A$ are fixed for any given MDP, it is clear that the heavy-tailed regret term will dominate for sufficiently large $T$, which matches the minimax regret order of $T$ for linear bandits \citep{shao2018almost}. 

We also note that there is an unavoidable $\rd \defeq r_{max} - r_{min}$ multiplicative factor, where $r_{max} = \max_{s,a} r(s,a)$ and $r_{min} = \min_{s,a}r(s,a)$. However, such a factor would also be present in the trivial generalization from bounded [0,1] reward distributions to $[r_{min}, \rmax]$.

Theorem \ref{thm:minimax} also immediately gives the following PAC bound on the regret:

\begin{corollary}
\label{cor:pac}
With probability at least $1-\delta$, the average per-step regret of \algo is at most $\lambda$ for any
\begin{equation}
\label{eq:pac_thresh}
\begin{split}
T \geq \max\Bigg(&4^2 \cdot 20^2 \frac{\rd^2 D^2S^2A}{\lambda^2}\log\left(\frac{40\rd DSA}{\delta\lambda}\right),\\
&\alpha \log\left(\frac{2SA}{\delta}\right) + 2\alpha\log\left(\frac{\alpha}{\delta}\right)\Bigg)
\end{split}
\end{equation}
where 
\begin{equation*}
    \alpha\defeq \frac{7c(4C_\epsilon + 2)^{\frac{1+\epsilon}{\epsilon}}v^{\frac{1}{\epsilon}}(SA)^{\frac{1}{\epsilon}}}{\lambda^{\frac{1+\epsilon}{\epsilon}}}.
\end{equation*}
\end{corollary}



\paragraph{Problem-dependent regret bound.} We also have a logarithmic regret bound dependent on the gap between the best and second-best policies, as in \ucrl. To simplify the statement of the theorem, we only consider the heavy-tailed regime, i.e. that in which equation \ref{eq:robust_conc} is satisfied. We use $\lambda$ to denote gaps in reward values, to differentiate from $\epsilon$ for the reward distribution.

\begin{theorem}
\label{thm:gap_dependent_bound}
For any initial state $s\in\mathcal{S}$, any $T$ satisfying equation \ref{eq:robust_conc}, and any $\lambda>0$, with probability at least $1-3\delta$ the regret of \algo is 
\begin{equation*}
    \Delta(M, s, T) \leq 7c\iota\left(4C_\epsilon+2\right)^{\frac{1+\epsilon}{\epsilon}}\left(\frac{SA}{\lambda}\right)^{1/\epsilon} +\lambda T
\end{equation*}
Furthermore, let $g\defeq \rho^*(M) - \max_{s\in S}\max_{\pi:\mathcal{S}\rightarrow\mathcal{A}} \left\{\rho(M,\pi,s):\rho(M\pi,s) > \rho^*(M) \right\}$
be the gap between the average reward of the best and second-best policies. Then, the expected regret of \algo (with parameter $\delta\defeq\frac{1}{3T}$) is bounded by 
\begin{equation*}
\begin{split}
    \mathbb{E}[\Delta(M, s, &T)] \leq 7c\iota\left(4C_\epsilon+2\right)^{\frac{1+\epsilon}{\epsilon}}\left(\frac{2SA}{g}\right)^{1/\epsilon}
    +\\ 
    &\sum_{s,a} \left\lceil 1 + \log_2\left(\max_{\pi:\pi(s)=a} T_{\pi}\right)\right\rceil \max_{\pi:\pi(s)=a} T_\pi 
\end{split}
\end{equation*}
\end{theorem}

Theorem \ref{thm:gap_dependent_bound} highlights that as $\epsilon$ becomes arbitrarily small, the problem-dependent regret bound goes to infinity. In theorem \ref{thm:lower_bound}, we show that this dependency is inevitable.


\paragraph{Regret under changing MDP.} \algo is designed to be robust to heavy-tailed reward distributions. In this section, we consider an orthogonal notion of robustness. Namely, if the MDP is allowed to change up to $\ell-1$ times after the initial MDP, under the condition that the diameter never exceeds $D$, we have the following upper bound on the regret of \algo.

\begin{theorem}
\label{thm:changing}
Suppose $\epsilon < 1$ and $T \in O\left(\frac{D^{\frac{2+2\epsilon}{1-\epsilon}}S^{\frac{2\epsilon}{1-\epsilon}}}{A}\right)$
Restarting \algo with parameter $\frac{\delta}{\ell^2}$ at steps $\left\lceil\frac{i^{(1+2\epsilon)/\epsilon}}{\ell^{(1+\epsilon)/\epsilon}} \right\rceil$ for $i=1,2,3,\ldots$, the regret of \algo is upper bounded by 
\begin{equation}
\rd\ell^{\frac{\epsilon}{1+2\epsilon}}T^{\frac{1+\epsilon}{1+2\epsilon}}(SA)^{\oeps}    
\end{equation}
with probability at least $1-\delta$.
\end{theorem}

We note that this theorem skips the case where $\epsilon=1$, which recovers the result of Theorem 6 in \citet{jaksch2010near} up to a constant multiplicative factor. 


\subsection{\algoq Results}

We first state the following regret bounds for \algoq with UCB-Hoeffding and UCB-Bernstein:


\begin{theorem}
\label{thm:hqregret}
In the finite-horizon episodic MDP setting, the regret of \algoq with UCB-Hoeffding is $\Theta\left(r_{max}H^2\sqrt{SAT} + H^2(SA\iota)^{\eeps}T^{\oeps}\right)$.
\end{theorem}

\begin{theorem}
\label{thm:hqbregret}
In the finite-horizon episodic MDP setting, the regret of \algoq with UCB-Bernstein is 
\begin{equation*}
\begin{split}
    \Theta\Bigl(&\sqrt{H^3\rmax^3 SAT\iota} + H^2(SA\iota)^{\eeps}T^{\oeps}\\ 
    + &\sqrt{H^9\rmax^2u^\oeps S^3A^3\iota^3} + \sqrt{H^{\frac{1+4\epsilon}{\epsilon}}\rmax S^2A^2\iota^2}\\
    + &H^{\frac{1+3\epsilon}{\epsilon}}\sqrt{S^3A^3\iota^4\epsilon}\Bigr).
\end{split}
\end{equation*}
\end{theorem}

Again, we note that that for sufficiently large $T$, the heavy-tailed regret term dominates; i.e. learning the rewards is more difficult than learning the transition probabilities. As expected, UCB-Bernstein exploration offers no improvement on the asymptotic regret. 


Notably, aside from the simple scaling due to $r_{max}$, our algorithm only necessitates another additive factor corresponding to the second term in the regret bounds (UCB-Bernstein also has some terms independent of $T$). Interestingly, our analysis shows that even though $Q$-learning does not explicitly perform mean estimation and is not amenable to all robust mean estimation techniques, the truncation technique can still be applied. This indicates that truncation can potentially be used to make a broad class of optimistic confidence-interval based algorithms robust in the heavy-tailed setting.

Furthermore, although this setting is slightly different from that of \algo, it is clear that the dependence on $SA$ is better for \algoq by a factor of $(SA)^\frac{1-\epsilon}{1+\epsilon}$. As we will show, the dependence of this bound on $S, A, T$ is tight. Finally, as for \algo, this minimax regret bound also naturally translates into a PAC bound.


\subsection{Lower bound}
We prove the following lower bound on the expected regret for any algorithm in this heavy-tailed MDP setting.
\begin{theorem}
\label{thm:lower_bound}
For any fixed $T$ and algorithm, there exists a communicating MDP $M$ with diameter $D$ such that the expected regret of the algorithm is $\Omega\left((SA)^{\eeps}(T)^{\oeps}\right)$. In the finite-horizon episodic setting, there exists a MDP such that the expected regret is $\Omega\left(H(SA)^{\eeps}(T)^{\oeps}\right)$.
\end{theorem}

\begin{figure}[t!]
    \centering
    \begin{tikzpicture}[-latex, auto, node distance = 3.5 cm and 3.5 cm, on grid, semithick, state/.style={circle, draw, minimum width = 1.25cm}]
      \node[state] (A) {$s_0$};
      \node[state] (B) [right = of A] {$s_1$};
      \path (B) edge [bend right = 20] node[above = 0.1 cm] {$\delta$} (A);
      \path (A) edge [bend right = 15] node[above = 0.1 cm] {$\delta$} (B);
      \path (A) edge [bend right = 40, dashed] node[below = 0.1  cm] {$\delta+\lambda$} (B);
      \draw [->, dashed] (A) to[in=270, out=200, looseness=3.5] node[below left = 0.03 cm and 0.03 cm] {$1-\delta-\lambda$} (A);
      \draw [->] (A) to[in=90, out=160, looseness=3.5]  node[above left = 0.03 cm and 0.03 cm] {$1-\delta$} (A);
      \draw [->] (B) to[in=310, out=50, looseness=4.0] node[right = 0.05 cm] {$1-\delta$} (B);
    \end{tikzpicture}
    \caption{Two-state MDP for lower bound.}
    \label{fig:lower_bound}
\end{figure}
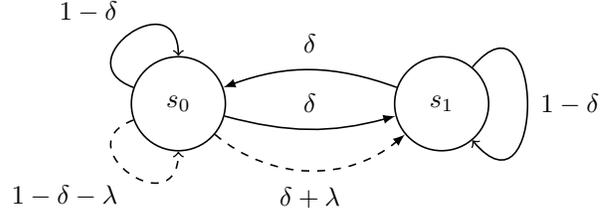

Intuitively, the proof of this theorem follows by combining the lower bound constructions for (problem-independent) heavy-tailed bandits in \citet{bubeck2013bandits} and finite-diameter MDPs in \citet{jaksch2010near}. The latter considers, without loss of generality, a simple two-state MDP (shown in figure \ref{fig:lower_bound}) in which for all actions except a single optimal action $a^*$, the probability of transitioning to the other state is $\delta$. However, when action $a^*$ is taken in $s_0$, the agent has probability $\delta+\lambda$ of transitioning to $s_1$. 

Now, recall that in the construction in \citet{bubeck2013bandits}, there are $K-1$ arms have Bernoulli reward scaled by $\frac{1}{\gamma}$ with parameter $\gamma^{1+\epsilon}-\Delta\gamma$ and one optimal arm with $\frac{1}{\gamma}$-scaled Bernoulli with parameter $\gamma^{1+\epsilon}$, where $\Delta$ is the gap in the expected reward of the optimal arm and that of other arms, and $\gamma \defeq (2\Delta)^{1/\epsilon}$. Then, directly setting $\Delta=(K/T)^{\eeps}$ yields a $K^{\eeps}T^{\oeps}$ regret bound. In our setting, noting that a two-state MDP with unknown transition probabilities is equivalent to learning a Bernoulli reward \citep{osband2016lower}, translating the bandit lower-bound construction to the MDP in figure \ref{fig:lower_bound} yields a lower bound of $\Theta\left((SA)^{\eeps}(T)^{\oeps}\right)$. 

For the episodic setting, we can concatenate $H$ of these MDPs together into a chain, as in \citet{jin2018q}. Hence, the total number of samples available for each step $h\leq H$ is now $T/H$, and we can rescale the bound accordingly.
\section{Experimental Results}
\label{sec:experiment}

\subsection{Tabular MDPs}

\begin{figure}[t!]
    \centering
    \begin{tikzpicture}[-latex, auto, node distance = 1.0cm and 1.6 cm, on grid, semithick, state/.style={circle, draw, minimum width = 1.1cm}]
      \node[state] (A) {$s_0$};
      \begin{scope}[node distance = 1.25 cm and 1.35 cm]
      \node[state] (B) [above right = of A] {$s_1$};
      \node[state] (F) [below right = of A] {$s_{l+1}$};
      \end{scope}
      \node[state] (C) [right = of B] {$s_2$};
      \node[] (D) [right = of C] {$\cdots$};
      \node[state] (E) [right = of D] {$s_l$};
      \node[state] (G) [right = of F] {$s_{l+2}$};
      \node[] (H) [right = of G] {$\cdots$};
      \node[state] (I) [right = of H] {$s_{2l+1}$};
      \path (A) edge [dashed] node[below = 0.1 cm] {$1$} (B);
      \path (B) edge [bend right=40, dashed] node[above=0.05 cm] {$1$} (A);
      \path (B) edge node[below = 0.05 cm] {$p$} (C);
      \path (C) edge node[below = 0.05 cm] {$p$} (D);
      \path (D) edge node[below = 0.05 cm] {$p$} (E);
      \path (A) edge node[above = 0.1 cm] {$1$} (F);
      \path (F) edge [bend left=40, dashed] node[below=0.05 cm] {$1$} (A);
      \path (F) edge node[below = 0.05 cm] {$p$} (G);
      \path (G) edge node[below = 0.05 cm] {$p$} (H);
      \path (H) edge node[below = 0.05 cm] {$p$} (I);      
      \path (C) edge [bend right = 40] node[above = 0.05 cm] {$1-p$} (B);
      \path (C) edge [bend left = 50, dashed] node[below = 0.05 cm] {$1$} (B);
      \path (D) edge [bend right = 40] node[above = 0.05 cm] {$1-p$} (C);
      \path (D) edge [bend left = 50, dashed] node[below = 0.05 cm] {$1$} (C);
      \path (E) edge [bend right = 40] node[above = 0.05 cm] {$1-p$} (D);
      \path (E) edge [bend left = 50, dashed] node[below = 0.05 cm] {$1$} (D);
      \path (G) edge [bend right = 40] node[above = 0.05 cm] {$1-p$} (F);
      \path (G) edge [bend left = 50, dashed] node[below = 0.05 cm] {$1$} (F);
      \path (H) edge [bend right = 40] node[above = 0.05 cm] {$1-p$} (G);
      \path (H) edge [bend left = 50, dashed] node[below = 0.05 cm] {$1$} (G);
      \path (I) edge [bend right = 40] node[above = 0.05 cm] {$1-p$} (H);
      \path (I) edge [bend left = 50, dashed] node[below = 0.05 cm] {$1$} (H);     
      \draw [] (E.25) arc (120:-120:3mm) node[below right = 0.1 cm and 0.1 cm] {$p$} (E);
      \draw [] (I.25) arc (120:-120:3mm) node[below right = 0.1 cm and 0.1 cm] {$p$} (I);
    \end{tikzpicture}
    \caption{DoubleChain MDP. The dashed and solid lines indicate the transition probabilities of the two actions.}
    \label{fig:double_chain}
\end{figure}
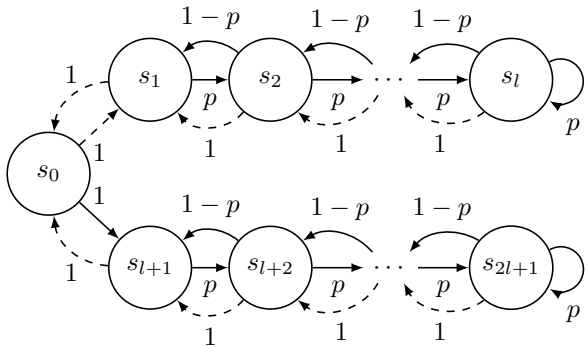

We first evaluate \algo and \algoq on synthetic MDPs (SixArms \citep{strehl2008analysis} and DoubleChain, shown in figure \ref{fig:double_chain}), comparing against classical light-tailed baselines: Gaussian Posterior Sampling (PSRL), Q-learning, and \ucrl. For PSRL, it is non-trivial to design a variant that accomodates general heavy-tailed reward distributions.  However, on light-tailed distributions, it is state-of-the-art \citep{osband2017posterior}. Hence, we include it as a benchmark in order to explicitly demonstrate that an algorithm designed for light-tailed rewards may perform poorly in the heavy-tailed setting. For \ucrl, to ensure a fair comparison, we use the heavy-tailed variant described in section \ref{sec:algorithm}. As discussed, this algorithm quickly devolves into random exploration, i.e. uniform selection of actions, and hence performs extremely poorly in our experiments. We therefore omit heavy-tailed \ucrl from our following results. For all of the above algorithms, we scale the confidence intervals by factors searched within $[1e-2, 1]$, as is standard practice \citep{lu2019optimal}. 

For \algo and \algoq, we used the truncated mean estimator, and $\epsilon=0.05$ in all experiments. For \algoq, we use the UCB-Hoeffding bonus.

For the reward distributions, we used symmetric Levy $\alpha$-stable distributions, for which the heaviness of the tail is controlled by $\alpha$ and the mean can be arbitrarily specified. In particular, the $(1+\epsilon)$-th moments of $\alpha$-stable distributions for $\epsilon < \alpha$ are bounded, i.e. these distributions are the ``heaviest'' possible in this setting \citep{dubey2019thompson}. Such distributions can be denoted as $\mathcal{L}(\mu, \alpha, \beta, \sigma)$, where $\mu$ is the mean, and $\alpha$, $\beta$, $\sigma$ the stable, skew, and shape parameters respectively. In all experiments, we set $\alpha=1.1$, and only consider such distributions with $\sigma=1$ that are symmetric (i.e. $\beta=0$). The full experimental details, including the specific reward distributions for each state-action pair, can be found in the appendix.

\paragraph{Results.} We report the total cumulative rewards for each enviroment, averaged over 30 random seeds, in figures \ref{sub:dc_all} and \ref{sub:sa_all}. \algo and \algoq clearly outperform PSRL and $Q$-learning on both MDPs. As expected, the latter two converge too quickly to a suboptimal policy, and never succeed in identifying the best transition. On the other hand, \algo and \algoq are able to perform enough exploration to discover better policies due to their appropriately-constructed confidence regions/bonuses. This is evidenced by the early iterations, in which PSRL and Q-learning initially achieve higher reward but ultimately fail to identify the best policies. 


Overall, these experiments underscore two complementary failures of light-tailed algorithms in heavy-tailed settings: they may converge too quickly to a suboptimal policies, and thus will fail to allocate enough samples to the optimal, heavy-tailed arms because of 1. the brittleness of the empirical mean estimator and 2. the too-tight confidence intervals.

\begin{figure*}[t!]
    \centering
    \begin{subfigure}{0.325\textwidth}
        \includegraphics[width=\linewidth]{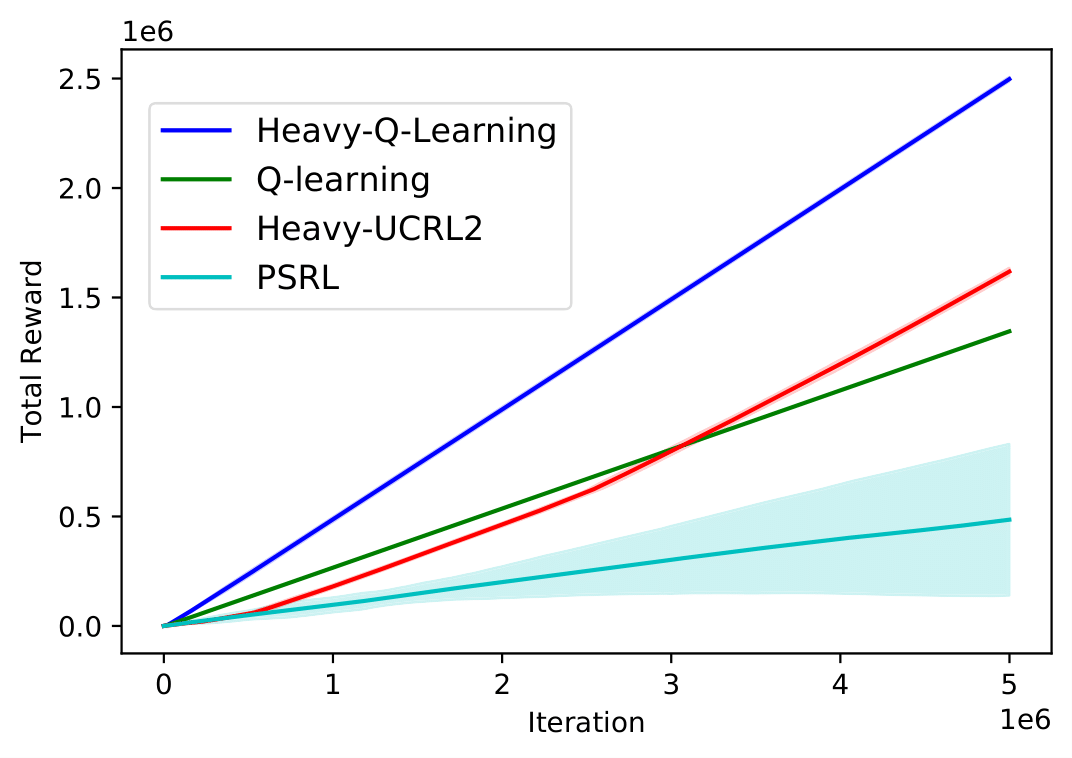}
        \caption{DoubleChain}
        \label{sub:dc_all}
    \end{subfigure}
    \begin{subfigure}{0.325\textwidth}
        \includegraphics[width=\linewidth]{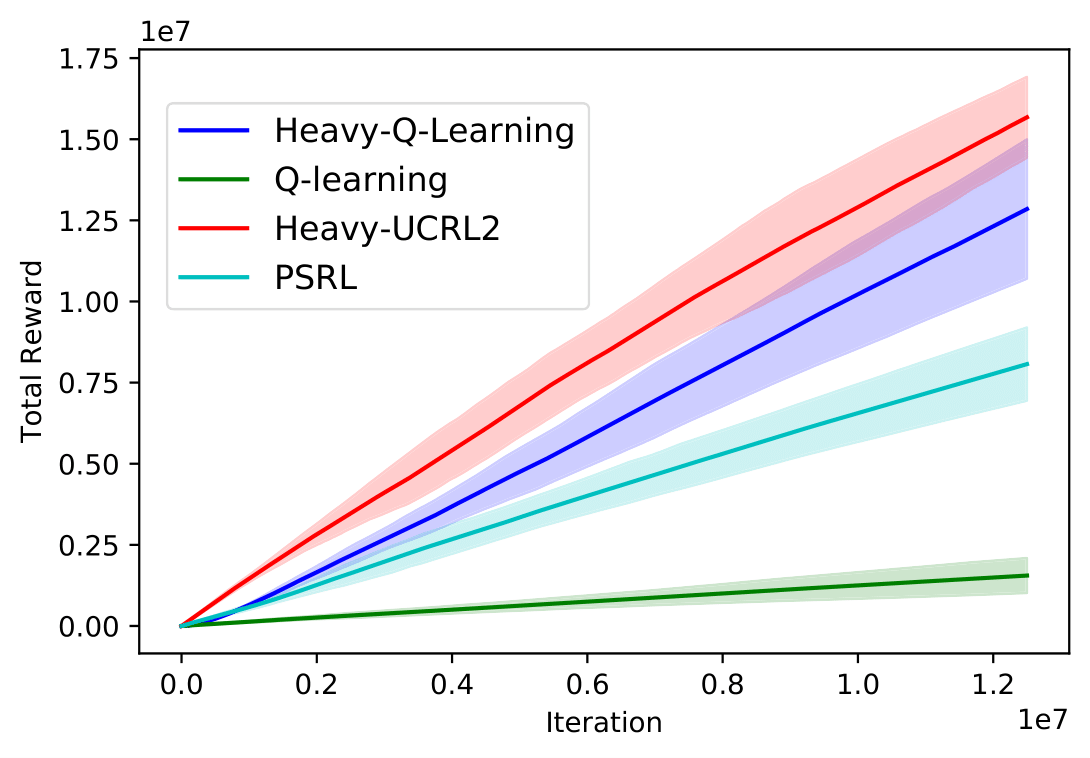}
        \caption{SixArms}
        \label{sub:sa_all}
    \end{subfigure}    
    \begin{subfigure}{0.325\textwidth}
        \includegraphics[width=\linewidth]{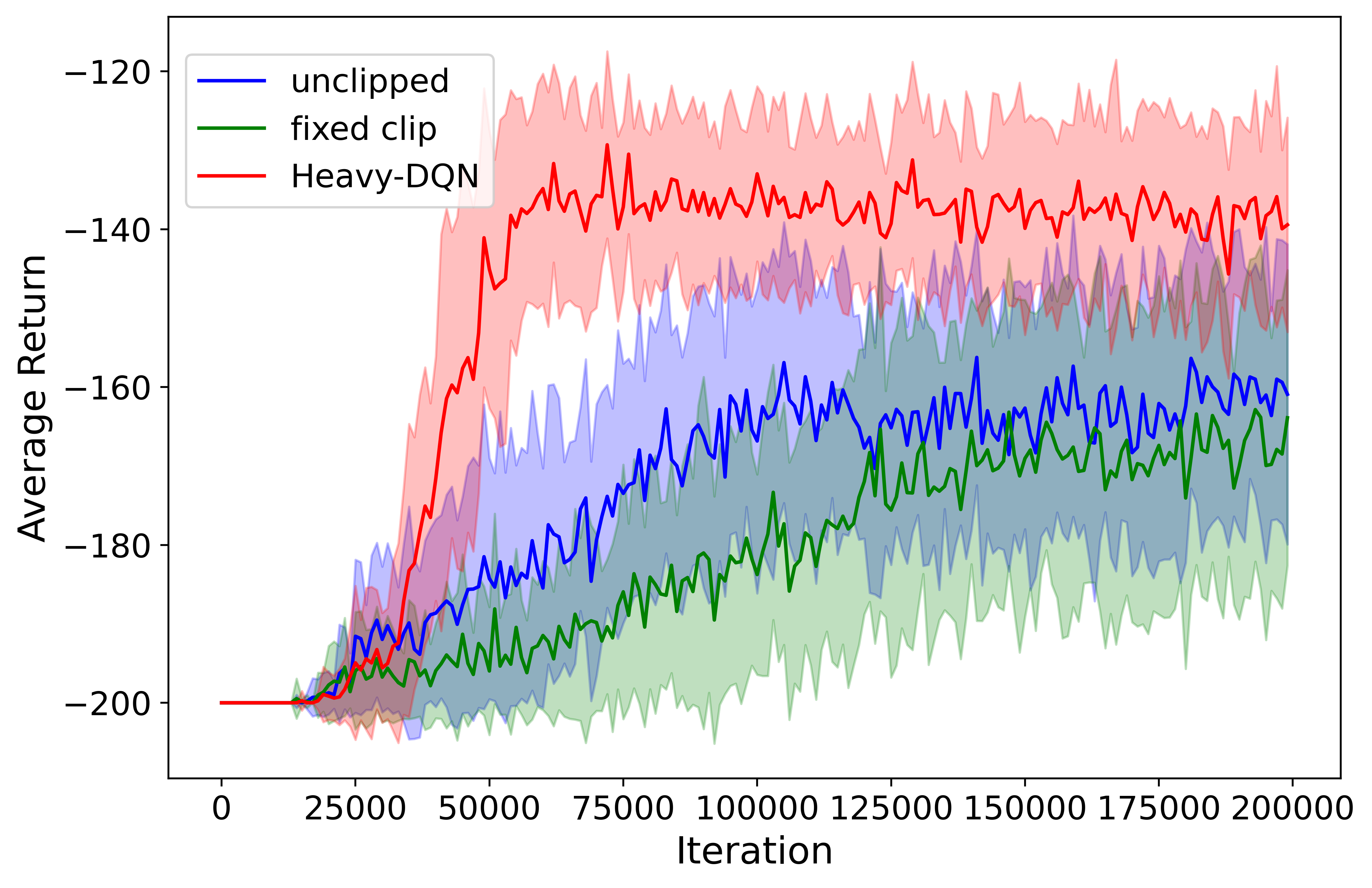}
        \caption{MountainCar-v0}
        \label{sub:mountaincar}
    \end{subfigure}
    \caption{Mean cumulative returns averaged over independent trials, with $\pm1$ standard deviation shaded.}
    \label{fig:exp_res}
\end{figure*}

\subsection{Deep Reinforcement Learning Experiments}
We tested our deep RL algorithm Heavy-DQN on the classical control environment of MountainCar. We perturbed the standard rewards with heavy-tailed noise following a symmetric $\alpha$-stable distribution with $\alpha=1.1$ and scale parameter $5$. The algorithm conservatively uses $\epsilon=0.05$. Our code is based off of the Dopamine reinforcement learning library \citep{castro18dopamine}, and the remaining hyperparameters were untuned from the defaults (restated in the appendix).

We compare our adaptive truncation method vs. a naive fixed clipping baseline, as well as an unclipped baseline, where the only hyperparameters tuned are the truncation coefficient $c$ as well as the clipping range. All other shared hyperparameters are held constant between algorithms. To ensure fair comparison, each algorithm is run over a fixed set of 24 random seeds. 

\paragraph{Results.} We plot the mean training return with $\pm1$ standard deviations shaded in figures \ref{sub:mountaincar}. The best $c$ was 0.015, and the best fixed clipping range was $[-5, +5]$. As expected, our adaptive truncation method outperforms both the fixed clipping and unclipped baseline. Heavy-DQN is able to quickly converge to a near-optimal policy, whereas the others become stuck in sub-optimal ones. Overall, these experiments demonstrate that truncation has promise in robustifying deep reinforcement learning algorithms with respect to reward noise. 


\section{Conclusion \& Discussion}
\label{conclusion}
We introduced the heavy-tailed reinforcement learning problem and proposed two algorithms, \algo and \algoq, that are provably efficient in this setting. More importantly, our results indicate that a broad-class of confidence-based RL algorithms are amenable to simple techniques from robust statistics. For example, state-of-the-art algorithms such as UCFH and UCBVI are likely compatible with the ideas behind \algo and \algoq \citep{dann2015sample,azar2017minimax}. However, it is unclear how such ideas can be applied to Bayesian algorithms, such as PSRL \citep{agrawal2017optimistic}.


All of our algorithms require a priori knowledge of $\epsilon$. However, this is a drawback of all bandit algorithms with heavy-tailed noise, and probably cannot be avoided in the absence of further assumptions \citep{shao2018almost}. We leveraged the median-of-means and truncated mean estimators for their theoretical guarantees, but both have their drawbacks in practical usage. The truncated mean estimator is simple to use and requires no overhead, but the confidence interval scaling factor can be large if the magnitude of the mean rewards are large, due to its dependence on the raw $(1+\epsilon)$-th moment. The median-of-means estimator is typically more robust and uses the centered moment bound, but has suboptimal constants. It also requires significantly larger time/space complexity (an additional multiplicative $O(T)$ factor for both).

Overall, we showed that our algorithms bridged theory and practice, providing significant empirical evidence that \algo and \algoq outperform classical RL algorithms on MDPs with heavy-tailed rewards even without relaxing their theoretical assumptions. We also proposed Heavy-DQN, an extension of \algoq to the deep reinforcement learning setting. Our results demonstrate that despite its simplicity, this adaptive clipping method is effective in stabilizing DQN under heavy-tailed reward noise. There are a number of straightforward extensions that may improve Heavy-DQN's performance and help it generalize to higher-dimensional environments, such as learned hash codes \citep{tang2017exploration}. 

In general, we proposed noisy rewards as a problem setting with real-world relevance. This, and its connections to clipping methods that are prevalent in RL application implementations, have been under-explored thus far and merit further study.



\bibliographystyle{plainnat}
\bibliography{references}

\appendix
\onecolumn

\title{Supplementary Material for "Robust Reinforcement Learning with Heavy-tailed Rewards"}



\section{Proof of Theorem \ref{thm:minimax}}
\label{app:minimax}



\subsection{Preliminaries}
\paragraph{Overview of \algo proof strategy.} Our regret analysis for \algo follows the standard strategy for optimistic algorithms \citep{osband2017posterior, jaksch2010near}. We first bound the regret in terms of the regret incurred in each episode and the random deviation of the rewards. We then consider the regret from episodes in which the confidence regions do not hold, i.e. the true MDP lies outside the confidence region. Finally, we bound the regret when the confidence regions do hold. This regret can be decomposed into the sum of three terms: the regret due to the error in transition probabilities, regret due to the error in rewards, and the regret due to extended value iteration.

The main differences in the analysis are due to the unboundedness and heavy-tailed nature of the rewards. From section 3, we show that for reward distributions with bounded $(1+\epsilon)$-th moments ($\epsilon \in (0, 1]$), we have concentration bounds similar to that of Hoeffding's inequality for arbitrary distributions over $[0, 1]$. We also leverage the boundedness of the 1st moment (i.e. the mean) to bound the regret when the confidence intervals fail. There is also an unavoidable $\rd \defeq r_{max} - r_{min}$ multiplicative factor, where $r_{max} = \max_{s,a} r(s,a)$ and $r_{min} = \min_{s,a}r(s,a)$. Note that although the observed rewards $r_t$ are unbounded, the means $\bar{r}$ necessarily are, since by assumption the $(1+\epsilon)$-th moment and hence the mean are bounded.


\paragraph{Convergence of extended value iteration for heavy-tailed rewards.} Since \algo relies on extended value iteration, we need to show that it works in our heavy-tailed setting. Fortunately, since extended value iteration uses fixed (optimistic) reward values, Theorem 7 in \citep{jaksch2010near} holds identically. We restate it below for completeness:

\begin{theorem}
\label{thm:evt}
Let $\mathcal{M}$ be the set of MDPs with state space $\mathcal{S}$, action space $\mathcal{A}$, transition probabilities $\tilde{p}(\cdot|s,a)$ and mean rewards $\tilde{r}(s,a)$ that satisfy eqs. \ref{eq:p_interval} and \ref{eq:r_interval}, and $u_i(s)$ denote the state value of $s$ at iteration $i$. Then, if $\mathcal{M}$ contains at least one communicating MDP, extended value iteration converges. Furthermore, stopping extended value iteration when
\begin{equation*}
    \max_{s\in\mathcal{S}}\{u_{i+1}(s) - u_i(s)\} - \min_{s\in\mathcal{S}}\{u_{i+1}(s) - u_i(s)\} < \epsilon
\end{equation*}
the greedy policy with respect to ${u}_i$ is $\epsilon$-optimal.
\end{theorem}

\subsection{Concentration of reward.} Conditioned on the state-action counts $N(s,a)$, the rewards $r_t$ are independent random variables. Hence, from concentration results for robust mean estimators, we have
\begin{equation}
P\left(\sum_{t=1}^T r_t \leq \sum_{s,a} N(s,a)\bar{r}(s,a) - C_T \middle| (N(s,a))_{s,a}\right) \leq \left(\frac{\delta}{8T}\right)^{5/4} < \frac{\delta}{12T^{5/4}}
\end{equation}
where $C_T = (v T)^{\frac{1}{(1+\epsilon)}}\left(\frac{5c}{4}\log{\left(\frac{8T}{\delta}\right)} \right)^{\frac{\epsilon}{(1+\epsilon)}}$.

Letting $\Delta_k = \sum_{s,a} v_k(s,a)(\rho^* - \bar{r}(s,a))$ denote the regret for episode $k$, we can bound the total regret as:
\begin{equation}
\Delta(s_1, T) = T\rho^* - \sum_{t=1}^T r_t \leq \sum_k \Delta_k + C_T
\end{equation}
which holds with probability at least $1 - \frac{\delta}{12T^{5/4}}$.

\subsection{Regret from failing confidence regions} 

We first bound the regret due to failing confidence regions. The idea is the same as that in \cite{jaksch2010near}: we show that since our confidence intervals are sufficiently large, the probability that they fail is very small per timestep.

The following lemma bounds the probability that $M \notin \mathcal{M}_k$:

\begin{lemma}
For any $t\geq 1$, 
\begin{equation}
P\left(M\notin\mathcal{M}_t\right) < \frac{\delta}{15t^6}
\end{equation}
\end{lemma}

\begin{proof}
This lemma and proof are essentially identical to Lemma 17 in \citep{jaksch2010near}. Specifically, the probability that the confidence intervals in eqs. \ref{eq:r_interval} and \ref{eq:p_interval} are violated are bounded by $\frac{\delta}{60t^7SA}$ and $\frac{\delta}{20t^7 SA}$ respectively. Then, a union bound over all timesteps and state-action pairs yields the desired result.
\end{proof}

Now, it follows identically from \citet{jaksch2010near} that $\sum_{k=1}^m \Delta_k \mathds{1}_{M\notin \mathcal{M}_k} \leq \sqrt{T}$ with probability at least $1-\frac{\delta}{12T^{5/4}}$. In our setting, since the rewards have bounded means, the regret in this scenario can be bounded by $\rd\sqrt{T}$.

\subsection{Regret for true MDP \texorpdfstring{$M \in \mathcal{M}_k$}{M in M_k}.} We now bound the regret in each episode $k$ assuming that our confidence set contains the true MDP. 

In both \ucrl and \algo, the condition for stopping extended value iteration is identical. Hence, we similarly have by Theorem \ref{thm:evt}:
\begin{equation}
\begin{split}
\Delta_k \leq \sum_{s,a} v_k(s,a)\left(\rho^* - \bar{r}(s,a)\right)
\leq \sum_{s,a} v_k(s,a)\left(\widetilde{\rho}_k-\bar{r}(s,a)\right) + \sum_{s,a} \frac{v_k(s,a)}{\sqrt{t_k}}
\end{split}
\end{equation}

Following the notation of \citet{jaksch2010near}, let $\widetilde{\textbf{P}}_k$ denote the transition matrix of the optimistic MDP $\tilde{M}$, $\textbf{P}_k$ the transition matrix of the true MDP, $\textbf{v}_k$ the row vector corresponding to the visit counts for each state (and corresponding action chosen by $\widetilde{\pi}_k$), and $\textbf{w}_k$ the vector defined as $w_k(s) \defeq u_i(s) - \frac{\min_s{u_i(s)} + \max_s{u_i(s)}}{2}$. Then, the regret can be decomposed as follows:
\begin{equation}
\label{eq:episode_bound}
\Delta_k \leq \textbf{v}_k(\widetilde{\textbf{P}}_k - \textbf{P}_k)\textbf{w}_k + \textbf{v}_k(\textbf{P}_k - \textbf{I})\textbf{w}_k
+ 2 \sum_{s,a} v_k(s,a) v^{\oeps} \left(\frac{7c\log\left(\frac{2SAt_k}{\delta}\right)}{{\max\{1, N_k(s,a)\}}}\right)^{\eeps}
+ 2\sum_{s,a} \frac{v_k(s,a)}{\sqrt{t_k}}
\end{equation}
Only the third term is dependent on the distribution of $r$. The rest of the terms hold identically as in \cite{jaksch2010near} with one small modification: for any iteration $i$ in value iteration, the range of state values is bounded by $R_\Delta D$. To restate, the sum of the first and fourth terms is bounded by 
\begin{equation}
    \rd D \left(\sqrt{14S\log{\left(\frac{2AT}{\delta}\right)}} + 2\right) \left(\sqrt{2} + 1\right)\sqrt{SAT}
\end{equation}
and the second term is bounded by 
\begin{equation}
    \rd \left(D\sqrt{\frac{5}{2}T\log{\frac{8T}{\delta}}} + DSA\log_2\left(\frac{8T}{SA}\right) \right)
\end{equation}
To bound the third term, we require the following lemma, which generalizes Lemma 19 in \cite{jaksch2010near}:
\begin{lemma}
\label{lem:sequence}
For any sequence of numbers $z_1, \ldots, z_n$ with $0\leq z_k \leq Z_{k-1} := \max\{1, \sum_{i=1}^{k-1} z_i\}$, there exists some constant $C_\epsilon$ such that
\[\sum_{k=1}^n \frac{z_k}{Z_{k-1}^{\eeps}} \leq C_\epsilon {Z_n}^{\oeps}\]
\end{lemma}
\begin{proof}[Proof of Lemma \ref{lem:sequence}]
Let $C_\epsilon = \frac{2}{M}$, where $M = \sum_{k=1}^\infty \binom{\oeps}{k} = 2^{\oeps} - 1$. Note that $2 < C_\epsilon \leq 2(\sqrt{2} + 1)$.

We prove by induction on $n$. For the base case, we consider all $n$ such that $\sum_{k=1}^{n-1} z_k \leq 1$. Then $Z_k = 1$ for all $k<n$ and $\sum_{k=1}^{n}\frac{z_k}{Z_{k-1}^\eeps} = \sum_{k=1}^{n} z_k = \sum_{k=1}^{n-1} z_{k} + z_n \leq 1 + 1 < C_\epsilon Z_n^\oeps$.  

For the inductive step, assume the claim holds for some $n$. Then, we have
\begin{equation*}
    \begin{split}
        \sum_{k=1}^{n+1} \frac{z_k}{Z_{k-1}^{\eeps}} &\leq C_{\epsilon}Z_{n}^{\oeps} + \frac{z_{n+1}}{Z_{n}^{\eeps}}\\
        &= \frac{1}{2}C_\epsilon\left(Z_{n}^{\oeps} + \frac{z_{n+1}}{Z_{n}^{\eeps}} \sum_{k=1}^\infty \binom{\oeps}{k} \right) + \frac{1}{2}C_{\epsilon}Z_{n}^{\oeps}\\
        &= \frac{1}{2}C_\epsilon\left(Z_{n}^{\oeps} + \frac{1}{Z_{n}^{\eeps}} \sum_{k=1}^\infty \binom{\oeps}{k} z_{n+1}^{1-k} z_{n+1}^k\right) + \frac{1}{2}C_{\epsilon}Z_{n}^{\oeps}\\
        &\leq \frac{1}{2}C_\epsilon\left(Z_{n}^{\oeps} + \frac{1}{Z_{n}^{\eeps}} \sum_{k=1}^\infty \binom{\oeps}{k} Z_{n}^{1-k} z_{n+1}^k\right) + \frac{1}{2}C_{\epsilon}Z_{n}^{\oeps}\\
        &\leq \frac{1}{2}C_\epsilon\left(Z_{n}^{\oeps} + \sum_{k=1}^\infty \binom{\oeps}{k} Z_{n}^{\oeps-k} z_{n+1}^k\right) + \frac{1}{2}C_{\epsilon}Z_{n}^{\oeps}\\
        &= \frac{1}{2}C_\epsilon\sum_{k=0}^\infty \left(\binom{\oeps}{k} Z_n^{\oeps - k} z_{n+1}^k\right) + \frac{1}{2}C_{\epsilon}Z_{n}^{\oeps}\\
        &= \frac{1}{2}C_\epsilon(Z_n + z_{n+1})^{\oeps} + \frac{1}{2}C_{\epsilon}Z_{n}^{\oeps}\\
        &\leq C_\epsilon Z_{n+1}^{\oeps}
    \end{split}
\end{equation*}

which completes the induction.
\end{proof}

Via this lemma and Jensen's inequality, we have
\begin{equation}
    \sum_{s,a}\sum+{k}\frac{v_k(s,a)}{(\max\{1, N_k(s,a)\})^{\eeps}} \leq C_{\epsilon}\sum_{s,a}N_{s,a}^{\oeps} \leq C_\epsilon (SAT)^{\oeps}
\end{equation}

which directly yields the following bound on the third term:
\begin{equation}
2C_\epsilon v^{\oeps} \left(7c\log\left(\frac{2SAT}{\delta}\right)\right)^{\eeps} \left(SAT\right)^{\eeps}
\end{equation}

Combining all of the above yields the following bound on the regret:
\begin{equation}
\begin{split}
    \Delta(T) \leq C_T &+ \sqrt{T} + \rd \left(D\sqrt{\frac{5}{2}T\log{\frac{8T}{\delta}}} + DSA\log_2\left(\frac{8T}{SA}\right) \right) \\
    &+ \rd D \left(\sqrt{14S\log{\left(\frac{2AT}{\delta}\right)}} + 2\right) \left(\sqrt{2} + 1\right)\sqrt{SAT}\\
    &+ 2C_\epsilon v^{\oeps} \left(7c\log\left(\frac{2SAT}{\delta}\right)\right)^{\eeps} \left(SAT\right)^{\oeps}
\end{split}
\end{equation}

We can simplify the bound using similar manipulations as in Appendix C.4 of \citet{jaksch2010near}. Noting that we can bound the heavy tailed terms by 
\begin{equation*}
    (2C_\epsilon + 1) v^{\oeps} \left(7c\log\left(\frac{2SAT}{\delta}\right)\right)^{\eeps}\left(SAT\right)^{\oeps}
\end{equation*}
Using $B_\epsilon \defeq (2C_\epsilon + 1) v^{\oeps} \left(7c\right)^{\eeps}$, we have the following bound on the regret:
\begin{equation}
    \Delta(s_0, T) \leq 20\rd DS\sqrt{AT\log\left(\frac{T}{\delta}\right)} + B_\epsilon  \left(\log\left(\frac{2SAT}{\delta}\right)\right)^{\eeps}\left(SAT\right)^{\oeps}
\end{equation}



\section{Proof of Corollary \ref{cor:pac}}
\label{app:cor}

\begin{proof}
We simply have to compute $T_0$ such that 
\begin{equation*}
    \frac{20\rd DS\sqrt{AT\log\left(\frac{T}{\delta}\right)}}{T} + \frac{(2C_\epsilon+1) v^{\oeps}(SAT)^{\oeps} \left(7c\log\left(\frac{2SAT}{\delta}\right)\right)^{\eeps}}{T} < \lambda
\end{equation*}
for all $T\geq T_0$.

Any $T_0$ satisfies the condition when both the first term and the second term are smaller than $\lambda / 2$. The first term comes directly from \cite{jaksch2010near}, replacing $\lambda$ with $\lambda / 2$. For the second term, the condition is equivalent to 
\begin{equation*}
\begin{split}
    T > \alpha\log\left(\frac{2SAT}{\delta}\right) = \alpha \log\left(\frac{2SA}{\delta}\right) + \alpha\log\left(\frac{T}{\delta}\right)
\end{split}
\end{equation*}         
Using the fact that $x > 2\log{x}$ for $x$ > 0, we see that $2\alpha\log\left(\frac{\alpha}{\delta}\right) >\alpha\log\left(\frac{T}{\delta}\right)$, from which the result follows immediately. 
\end{proof}

\section{Proof of Theorem \ref{thm:gap_dependent_bound}}
\label{app:log_bound}

The proof of this theorem follows as a consequence of this lemma:
\begin{lemma} 
\label{lem:suboptimal_ep}
Call an episode $\lambda$-bad if its average regret is larger than $\lambda$. Let $L_\lambda$ be the number of steps taken by \algo in $\lambda$-bad episodes up to step $T$. Then for any initial state $s\in\mathcal{S}$, any $T$ satisfying 
\begin{equation}
    \label{eq:heavy_cond}
    T \geq \frac{20\rd D^{\frac{2+2\epsilon}{1-\epsilon}} S^{\frac{2\epsilon}{1-\epsilon}}}{B_\epsilon A},
\end{equation}

and any $\lambda>0$, with probability of at least $1-3\delta$,
\begin{equation*}
    L_\lambda(T) \leq 7c\log\left(\frac{2SAT}{\delta}\right)\left(\frac{4C_\epsilon+2}{\lambda}\right)^{\frac{1+\epsilon}{\epsilon}}(SA)^{\frac{1}{\epsilon}}
\end{equation*}
\end{lemma}

\begin{proof}[Proof of Lemma \ref{lem:suboptimal_ep}]
The proof, like that of Theorems 4 and 11 in \cite{jaksch2010near}, draws heavily on that of Theorem \ref{thm:minimax}. Let $K_\lambda$ denote the random set corresponding to the episodes that are $\lambda$-bad. Then, the regret of these episodes is bounded with probability at least $1-2\delta$ by 
\begin{equation}
    \Delta'_{\lambda}(s, T) \leq (vL_{\lambda})^{\oeps} \left(c\log\left(\frac{T}{\delta}\right)\right)^{\eeps} + \sum_{k\in K_\lambda} \Delta_k \mathds{1}{M\in\mathcal{M}_k}
\end{equation}

We can leverage our previous bound on the regret within each episode (eq.\ref{eq:episode_bound}) with some small modifications. In particular, in this setting we need to bound $\sum_{k\in K_\epsilon}\sum_{s,a} \frac{v_k(s,a)}{{\max\{1, N_k(s,a)\}}^\eeps}$. This is proven in lemma \ref{lem:bad_ep_sum}. Leveraging this result, we have
\begin{equation}
\begin{split}
    \Delta_\epsilon'(s, T) &\leq (vL_{\lambda})^{\oeps} \left(c\log\left(\frac{T}{\delta}\right)\right)^{\eeps} + \rd D \left(\sqrt{14S\log{\left(\frac{2AT}{\delta}\right)}}\right)\sqrt{SAT}\\
    &+ \sum_{k\in K_\epsilon} \textbf{v}_k(\textbf{P}_k - \textbf{I})\textbf{w}_k \mathds{1}_{M\in\mathcal{M}_k} + C_\epsilon(L_\lambda SA)^{\oeps}
\end{split}
\end{equation}

As before, the third term can be bounded by 
\begin{equation*}
    \rd\left(2\sqrt{L_\lambda \log\left(\frac{T}{\delta}\right)} + DSA\log_2\left(\frac{8T}{SA}\right)\right)
\end{equation*}

So, we can simplify
\begin{equation}
    \Delta_{\lambda}'(s,T) \leq 20DS\sqrt{L_\lambda A \log\left(\frac{T}{\delta}\right)} + B_\epsilon \left(\log\left(\frac{2SAT}{\delta}\right)\right)^{\eeps}\left(L_{\lambda}SA\right)^{\oeps}
\end{equation}

Under the regime given by eq. \ref{eq:heavy_cond}, this bound is equivalent to $2B_\epsilon \left(\log\left(\frac{2SAT}{\delta}\right)\right)^{\eeps}\left(L_{\lambda}SA\right)^{\oeps}$. Finally, using the fact that $\lambda L_{\lambda} \leq \Delta_{\lambda}'(s, T)$ yields the desired bound.

\end{proof}

We now prove the following lemma:
\begin{lemma}
\label{lem:bad_ep_sum}
\begin{equation*}
\sum_{k\in K_{\lambda}}\sum_{s,a} \frac{v_k(s,a)}{{\max\{1, N_k(s,a)\}}^{\eeps}} \leq C_\epsilon \left(L_{\lambda}SA\right)^{\oeps}
\end{equation*}
\end{lemma}

\begin{proof}
The proof is identical to that of eq. 27 in \cite{jaksch2010near}. The only modification is that we have $d_k\defeq{\max\{1, N_k(s,a)\}}^{\eeps}$. Since the proof only uses the property that $d_k \leq d_l$ for $k\leq l$, it follows that we have
\begin{equation*}
    \sum_{k=1}^m \frac{v_k}{d_k}\mathds{1}_{k\in K_\lambda} \leq \sum_{k=1}^{m_\lambda} \frac{v_k'}{d_k}\leq C_\epsilon {\ell_\lambda}^{\oeps}
\end{equation*}
Summing over all state-action pairs and applying Jensen's inequality yields the desired result.
\end{proof}

Finally, we are ready to prove the main theorem.

\begin{proof}[Proof of Theorem \ref{thm:gap_dependent_bound}]
The first part of the theorem follows immediately from the preceding lemma. For the second part of the theorem, note that the expected regret in $\frac{g}{2}$-bad episodes is upper bounded by $7c\log\left(\frac{2SAT}{\delta}\right)\left(4C_\epsilon+2\right)^{\frac{1+\epsilon}{\epsilon}}\left(\frac{2SA}{g}\right)^{1/\epsilon} + 1$. The theorem statement then follows by identical argument as in \citet{jaksch2010near}.
\end{proof}

\section{Proof of Theorem \ref{thm:changing}}
\label{app:changing}

\begin{proof}[Proof of Theorem \ref{thm:changing}]
We first describe the intuition behind the proof. Essentially, if $T$ is sufficiently large, the regret of the \algo is $\tilde{O}\left((SAT)^{\oeps}\right)$. Hence, restarting \algo every $\left(\frac{T}{\ell}\right)^{\frac{1+\epsilon}{1+2\epsilon}}$ steps, the regret for the $\ell$ changing MDP periods is bounded by $R_\Delta\ell^{\frac{\epsilon}{1+2\epsilon}}T^{\frac{1+\epsilon}{1+2\epsilon}}$. Furthermore, since we restart \algo $T^{\frac{\epsilon}{1+2\epsilon}}\ell^{\frac{1+\epsilon}{1+2\epsilon}}$ times, the regret incurred in non-changing stages is also $\tilde{O}\left(\rd\ell^{\frac{\epsilon}{1+2\epsilon}}T^{\frac{1+\epsilon}{1+2\epsilon}}\right)$. Thus, the total regret is bounded by $\tilde{O}\left(\rd\ell^{\frac{\epsilon}{1+2\epsilon}}T^{\frac{1+\epsilon}{1+2\epsilon}}\right)$. 

Since the time horizon $T$ is not assumed to be known a priori, we use a variant of the ``doubling trick.'' Specifically, we restart \algo with parameter $\frac{\delta}{\ell^2}$ at steps $\left\lceil\frac{i^{(1+2\epsilon)/\epsilon}}{\ell^{(1+\epsilon)/\epsilon}} \right\rceil$ for $i=1,2,3,\ldots$, effectively dividing the algorithm into $n$ stages. For some fixed time horizon $T$, the total number of restarts $n$ is bounded by
\begin{equation}
\label{eq:n_bound}
    \ell^{\frac{1+\epsilon}{1+2\epsilon}}T^{\frac{\epsilon}{1+2\epsilon}} - 1 \leq n \leq\ell^{\frac{1+\epsilon}{1+2\epsilon}}T^{\frac{\epsilon}{1+2\epsilon}}
\end{equation}
The regret $\Delta_r$ incurred in the $\ell$ stages in which the MDP is restarted is bounded by $\rd$ times the total number of steps in these stages. This latter value is maximized when they occur in the last $\ell$ stages, which occur $T_{\ell}$ timesteps. Then we have
\begin{equation}
    \begin{aligned}
        T_{\ell}&\leq \frac{1}{\ell^{(1+\epsilon)/\epsilon}}\left((n+1)^{(1+2\epsilon)/\epsilon} - (n-\ell+1)^{(1+2\epsilon)/\epsilon}\right)\\
        &=\frac{1}{\ell^{(1+\epsilon)/\epsilon}}\sum_{k=0}^\infty \binom{\frac{1+2\epsilon}{\epsilon}}{k}n^{\frac{1+2\epsilon}{\epsilon} - k}\left(1-(1-\ell)^k\right)\\
        &\leq \frac{1+2\epsilon}{\epsilon}n^{\frac{1+\epsilon}{\epsilon}}\ell^{-\frac{1}{\epsilon}}\\
        &\leq \frac{1+2\epsilon}{\epsilon} \ell^{\frac{\epsilon}{1+2\epsilon}}T^{\frac{1+\epsilon}{1+2\epsilon}}
    \end{aligned}
\end{equation}
where in the second, third, and fourth (in)equalities we used the generalized binomial theorem, that the sum is dominated by $k=1$, and eq. \ref{eq:n_bound} respectively. The total regret $\Delta_r$ is then bounded by $\rd\ell T_{\ell} = \rd\frac{1+2\epsilon}{\epsilon} \ell^{\frac{1+\epsilon}{1+2\epsilon}}T^{\frac{1+\epsilon}{1+2\epsilon}}$.

We now consider the regret for the stages in which the MDP does not change. Letting $T_i\defeq \min(T , \tau_{i+1})-\tau_i$, we have by Theorem \ref{thm:minimax} that
\begin{equation}
\begin{aligned}
    \Delta(S_{\tau_i, T_i}) &\leq 2B_\epsilon \left(\log\left(\frac{2\ell^2 SAT_i}{\delta}\right)\right)^{\eeps}\left(SAT_i\right)^{\oeps}\\
    &\leq 2B_\epsilon \left(3\log\left(\frac{2SAT}{\delta}\right)\right)^{\eeps}\left(SAT_i\right)^{\oeps}
\end{aligned}
\end{equation}
with probability $\frac{\delta}{4\ell^2 T_i^{5/4}}$.

Summing over all stages $i=1,\ldots, n$, the total regret $\Delta_f$ is bounded by
\begin{equation}
    \begin{aligned}
        \Delta_f &=\sum_{i=1}^n \Delta(S_{\tau_i, T_i})\\
        &\leq 2B_\epsilon \left(3n\log\left(\frac{2SAT}{\delta}\right)\right)^{\eeps}\left(SAT\right)^{\oeps}\\
        &\leq 2B_\epsilon \ell^{\frac{\epsilon}{1+2\epsilon}}T^{\frac{1+\epsilon}{1+2\epsilon}}\left(3\log\left(\frac{2SAT}{\delta}\right)\right)^{\eeps}\left(SA\right)^{\oeps},
    \end{aligned}
\end{equation}
where the first inequality is due to Jensen's inequality. Finally, it follows from \citet{jaksch2010near} that the total probability is bounded by $1-\delta$.
\end{proof}


\section{Proof of Theorem \ref{thm:hqregret}}

Our proof follows that of Theorem 1 in \citet{jin2018q}. We first prove that $Q^k-Q^*$ is bounded for all $s,a,h,k$, then use that fact to recursively decompose the regret. We also use the similar notation: $s^k_h$ and $a^k_h$ denote the state and action taken at step $h$ in episode $k$, and define
\begin{equation}
    \alpha^0_t = \prod_{j=1}^t (1-\alpha_j), \qquad \ait = \alpha_i\prod_{j=i+1}^t (1-\alpha_j)
\end{equation}

We consider the expected regret, and for clarity we omit the expectation over the reward stochasticity. 

We first prove a useful auxiliary lemma about the learning rate, which generalizes Lemma 4.1.a in \citet{jin2018q}:

\begin{lemma}
\label{lem:heavy_conf_sum}
For all $\ait$, we have the following:
\begin{equation}
    t^{-\eeps} \leq \sum_{i=1}^t \ait i^{-\eeps} \leq 2t^{-\eeps}
\end{equation}
\end{lemma}

\begin{proof}
 We prove via induction on $t$. For the base case, $t=1$, and $\alpha^1_1 = 1$. Now assume the hypothesis holds for $t=1,\ldots, k-1$. For the lower bound, we have:
\begin{equation*}
    \begin{split}
        \sum_{i=1}^k \alpha^i_k i^{-\eeps} &= \alpha_k k^{-\eeps} + (1-\alpha_k)\sum_{i=1}^{k-1} \alpha^i_{k-1}i^{-\eeps}\\
        &\geq \alpha_k k^{-\eeps} + (1-\alpha_k)(k-1)^{-\eeps}\\
        &\geq \alpha_k k^{-\eeps} + (1-\alpha_k)k^{-\eeps}\\
        &= k^{-\eeps}
    \end{split}
\end{equation*}
and for the upper bound:
\begin{equation*}
    \begin{split}
        \sum_{i=1}^k \alpha^i_k i^{-\eeps} &= \alpha_k k^{-\eeps} + (1-\alpha_k)\sum_{i=1}^{k-1} \alpha^i_{k-1}i^{-\eeps}\\
        &\leq \alpha_k k^{-\eeps} + 2(1-\alpha_k)(k-1)^{-\eeps}\\
        &= \frac{H+1}{H+k}k^{-\eeps} + \frac{2k^{\oeps}}{H+k}\\
        &= 2k^{-\eeps} - \frac{H - 1}{H + k}k^{-\eeps}\\
        &\leq 2k^{-\eeps}
    \end{split}
\end{equation*}
which completes the induction.
\end{proof}

The following lemma, corresponding to lemmas 4.3 and C.4 in \citet{jin2018q}, upper and lower bounds the gap between the estimated $Q$-value function at each episode and and $Q^*$.

\begin{lemma}
\label{lem:hqgap}
Let $\beta'_t = 2\sum_{i=1}^t \ait b'_i \leq 16c\rmax\sqrt{H^3\iota / t} + 8H u^{\oeps}(\iota/t)^{\eeps}$ (by lemma \ref{lem:heavy_conf_sum}). Then for all $(s, a, h, k)$, $(Q^k_h - Q^*_h)(s,a)$ satisfies the following bound with probability at least $1-\delta$:
\begin{equation}
\label{eq:lemhqgap}
    0 \leq (Q^k_h - Q^*_h)(s,a) \leq \alpha^0_t H \rmax  + \sum_{i=1}^T \ait (V^{k_i}_{h+1}-V^*_{h+1})(s^{k_i}_{h+1}) + \beta'_t
\end{equation}
\end{lemma}

\begin{proof}
Since the reward is stochastic, the following identity holds for $Q^*$:
\begin{equation}
    Q^*_h(s,a) = \alpha^0_t Q^*_h(s, a) + \sum_{i=1}^t \ait\left[\bar{r}_h(s,a)+(\mathbb{P}_h-\hat{\mathbb{P}}^{k_i}_h)V^*_{h+1}(s,a) + V^*_{h+1}(s^{k_i}_{h+1})\right]
\end{equation}
Hence, we have
\begin{equation}
\label{eq:hqgap}
\begin{split}
    (Q^k_h - Q^*_h)(x,a) \leq \alpha^0_t (H\rmax - Q^*_h(s, a)) &+ \sum_{i=1}^t \ait \left[(V^{k_i}_{h+1}-V^*_{h+1})(s^{k_i}_{h+1}) + [(\hat{\mathbb{P}}^{k_i}_h - \mathbb{P}_h)V^*_{h+1}](s,a) + b_i\right]\\
    &+\sum_{i=1}^t \ait \left(r_h(s,a)\indl - \bar{r}_h(s,a)\right)
\end{split}
\end{equation}

We first consider the last term. Since $\ait \leq \frac{2H}{t}$, this quantity is bounded by $\frac{2H}{t} \sum_{i=1}^t (r_h \indl - \bar{r}_h
)(x,a)$. By Lemma 1 from \citet{bubeck2013bandits}, this is upper (and lower) bounded by $\pm 8Hu^{\oeps}\left(\frac{\iota}{t}\right)^{\eeps}$ with probability at least $1-\frac{\delta}{2SAT}$. A union bound over all $s, a, t$ implies that this uniformly holds with probability at least $1-\frac{\delta}{2}$.

It remains to modify the computation of the second term. Here, the setup is almost identical: we let $k_i\defeq \min\left(\{k\in [K] \,|\, k > k_{i-1} \land (s^k_h, a^k_h) = (s,a)\} \bigcup \{K+1\}\right)$. Then $(\mathds{1}[k_i\leq K]\cdot [(\hat{\mathbb{P}}^{k_i}_h - \mathbb{P}_h)V^*_{h+1}](s,a))^{\tau}_{i=1}$ is a martingale difference sequence with respect to the filtration consisting of the $\sigma$-field generated by all random variables up to step $h$ in episode $k_i$. The main difference in this setting is that $V^*$ is now uniformly bounded by $H\rmax$. Hence, a straightforward application of Azuma-Hoeffding followed by a union bound over $k_i\leq K$ results in the following bound, which holds with probability $1-\delta/{SAH}$:
\begin{equation}
\label{eq:mdsbound}
    \sum_{i=1}^\tau \alpha^i_{\tau} \cdot \mathds{1}[k_i\leq K]\cdot [(\hat{\mathbb{P}}^{k_i}_h - \mathbb{P}_h)V^*_{h+1}](s,a) \leq cHr_{\max}\sqrt{\sum_{i=1}^\tau (2\alpha^i_{\tau})^2\cdot\iota} \leq cr_{\max}\sqrt{\frac{H^3\iota}{\tau}}
\end{equation}
\end{proof}

We are now ready to prove theorem \ref{thm:hqregret}.

\begin{proof}
The proof follows identically as that of Theorem 2 in \citet{jin2018q} except for the following two differences: 1. $\beta'_t$ has an additional additive term, and 2. the martingale difference sequence involving $V^*-V^k$ is now bounded by $H\rmax$ instead of merely $H$. We compute the additional regret incurred for each term:

\begin{itemize}
    \item \textbf{Increase in $\sum_{h=1}^H\sum_{k=1}^K \beta_{n^k_h}$.} The additional regret incurred by $\beta'$ relative to $\beta$ is given by $$H^2u^{\oeps}\iota^{\eeps}\sum_{s,a}\sum_{n=1}^{N^K_h(s,a)} n^{-\eeps} \leq \Theta\left(H^2(SA\iota)^{\eeps}K^{\oeps}\right) \leq \Theta\left(H(SA\iota)^{\eeps}T^{\oeps}\right)$$ because the sum is maximized when $N^K_h(s,a) = \frac{K}{SA}$ for all $s,a$.
    \item \textbf{Regret scaling due to $V^*-V^k$.} By Azuma-Hoeffding, the regret is scaled by a $\rmax$ factor, i.e. $cH\rmax\sqrt{T\iota}$. This term is strictly dominated by the previous one.
\end{itemize}

Hence, the total regret is the same as that in \citet{jin2018q} scaled by $\rmax$ with an additional $\Theta\left(H^2(SA\iota)^{\eeps}T^{\oeps}\right)$ factor.
\end{proof}

\subsection{Proof of Theorem \ref{thm:hqbregret}}

\addtocounter{theorem}{-3}
\begin{theorem}
In the finite-horizon episodic MDP setting, the regret of \algoq with UCB-Bernstein is $$\Theta\left(\sqrt{H^3\rmax^3 SAT\iota} + \sqrt{H^{\frac{1+4\epsilon}{\epsilon}}\rmax S^2A^2\iota^2} + \sqrt{H^9\rmax^2u^\oeps S^3A^3\iota^3} + H^2(SA\iota)^{\eeps}T^{\oeps} + H^{\frac{1+3\epsilon}{\epsilon}}\sqrt{S^3A^3\iota^4\epsilon}\right)$$.
\end{theorem}

\begin{proof}[Proof of Theorem \ref{thm:hqbregret}]
We first re-define notation from \citet{jin2018q}. The variance operator for the next-state value is given by
\begin{equation}
    [\nsv](s,a) \defeq \mathbb{E}_{s'\sim\mathbb{P}_h(\cdot|s,a)} \left[V_{h+1}(s') - [\mathbb{P}_hV_{h+1}](s,a) \right]^2
\end{equation}
and the empirical variance estimate
\begin{equation}
    W_t(s,a,h) \defeq \frac{1}{t} \sum_{i=1}^t \left[V^{k_i}_{h+1}(s^{k_i}_{h+1}) - \frac{1}{t}\sum_{j=1}^t V^{k_j}_{h+1}(s^{k_j}_{h+1}) \right]^2
\end{equation}

To define the bonus, \citet{jin2018q} define
\begin{equation}
    \beta_t \defeq \min\left\{c_1\left(\sqrt{\frac{H}{t}\cdot(W_t(s,a,h) + H)\iota} + \frac{\sqrt{H^7SA}\cdot\iota}{t}\right), c_2\sqrt{\frac{H^3\iota}{t}}\right\}
\end{equation}
and let 
\begin{equation}
    b_1(s,a,h) \defeq \frac{\beta_1(s,a,h)}{2}, \qquad b_t(s,a,h) \defeq \frac{\beta_t(s,a,h)-(1-\alpha_t)\beta_{t-1}(s,a,h)}{2\alpha_t}
\end{equation}

We claim that the proof of Theorem 2 in \citet{jin2018q} holds in the stochastic heavy-tailed reward setting when using the following bonus:
\begin{equation}
\label{eq:hq_bonus}
\begin{split}
    \beta'_t \defeq &\min\left\{c_1\left(\sqrt{\frac{H\rmax}{t}\cdot(W_t(s,a,h) + H)\iota} + \frac{\sqrt{H^7\rmax SA}\cdot\iota}{t} + \frac{H^2\iota\sqrt{\rmax^3}}{t} + \frac{H^{\frac{1+2\epsilon}{\epsilon}}\iota\sqrt{SA\epsilon}}{t}\right), c_2\rmax\sqrt{\frac{H^3\iota}{t}}\right\}\\
    &+ 8H u^{\oeps}(\iota/t)^{\eeps}
\end{split}
\end{equation}

We proceed to calculate the regret following the steps of \citet{jin2018q}. Again, we use the same notation $\phi^k_h = (V^k_h-V^*_h)(s^k_h)$ and $\delta^k_h = (V^k_h - V^{\pi_k}_h)(s^k_h)$.

\paragraph{Coarse upper bound on $Q^k-Q^*$.} Note that that the corresponding bound for this term under UCB-Hoeffding (eq. \ref{eq:lemhqgap}) holds here identically. This fact will be used in the following to bound the gap between $W_t$ and $\nsv$. 

\paragraph{Bounding the gap between $W_t$ and $\nsv$} We first translate lemma C.7 \citet{jin2018q} to our setting, showing that for some non-negative weight vector $w=(w_1, \ldots, w_k)$, the sum $\sum_{k=1}^K \phi^k_h$ is bounded by $\sum_{k=1}^K (Q^k_h-Q^*_h)(s^k_h, a^k_h)$. In the bounded $[0,1]$ reward setting, this bound can be computed as $O(SA\|w\|_\infty\sqrt{H^5\iota}+\sqrt{SA\|w\|_1\|w\|_\infty H^5\iota})$. In our setting, these terms are scaled by $\rmax$, plus an additional term due to our modified $\beta_t'$. The increase due to this term is given by
\begin{equation}
\begin{split}
    O(H) \cdot \sum_{k=1}^K w_k Hu^\oeps\left(\frac{\iota}{n^k_h}\right)^\eeps &= O(H)\cdot \sum_{s,a}\sum_{i=1}^{N^K_h(s,a)}w_{k_i(s,a)} Hu^\oeps\left(\frac{\iota}{n^k_h}\right)^\eeps\\
    &\leq \Theta\left(H^2u^\oeps\iota^\eeps\right) \cdot \sum_{s,a} \|w\|_\infty \left(1 + \sum_{i=1}^{\left\lfloor \frac{\|w\|_1}{SA\|w\|_\infty}\right\rfloor}\left(\frac{1}{i}\right)^\eeps \right)\\
    &\leq \Theta\left(H^2u^\oeps\iota^\eeps \left(SA\|w\|_\infty + \left(SA\|w\|_\infty\right)^{\eeps}\|w\|_1^{\oeps}\right)\right)
\end{split}
\end{equation}

We can now apply this lemma to compute the increase in the bound between the empirical and actual variance. Recall that \citet{jin2018q} bound the gap between $P_1 \defeq [\nsv](s,a)$ and $P_4 \defeq W_t(s,a,h)$ via the triangle inequality on intermediate terms
\begin{equation}
    \begin{split}
        P_2 &\defeq \frac{1}{t}\sum_{i=1}^t \left[V^*_{h+1}(s^{k_i}_{h+1}) - [\pvo](s,a)\right]^2\\
        P_3 &\defeq \frac{1}{t}\sum_{i=1}^t \left[V^*_{h+1}(s^{k_i}_{h+1}) - \frac{1}{t}\sum_{j=1}^t V^*_{h+1}(s^{k_j}_{h+1}) \right]^2
    \end{split}
\end{equation}

We proceed to analyze the gaps $|P_1-P_2|$, $|P_2-P_3|$, $|P_3-P_4|$ in this setting. For $|P_1-P_2|$ and $|P_2-P_3|$, straightforward applications of the Azuma-Hoeffding inequality show that they are bounded by $cH^2\rmax^2\sqrt{\iota/t}$. For $|P_3 - P_4|$, we apply our preceding lemma: taking $w$ so that $w_{k_i} = \frac{1}{t}$ for $i=1,\ldots, t$ and $0$ otherwise, so that $\|w\|_1=1$ and $\|w\|_\infty=\frac{1}{t}$, yields
\begin{equation}
    |P_3 - P_4|\leq O\left(\rmax\sqrt{H^7\iota}\left(\frac{SA}{t}+\sqrt{\frac{SA}{t}}\right) + \frac{H^3u^\oeps\iota^\eeps SA}{t} + \frac{H^3u^\oeps (SA\iota)^\eeps}{t^\eeps} \right)
\end{equation}

Hence, the gap between the empirical variance and the actual variance for \algoq with UCB-Bernstein is upper bounded by
\begin{equation}
    \label{eq:vargap}
    \Theta\left(H^2\rmax^2\sqrt{\frac{\iota}{t}} + \rmax\sqrt{H^7\iota}\left(\frac{SA}{t}+\sqrt{\frac{SA}{t}}\right) + \frac{H^3u^\oeps\iota^\eeps SA}{t} + \frac{H^3u^\oeps (SA\iota)^\eeps}{t^\eeps}\right)
\end{equation}

\paragraph{Applying Freedman's inequality.} Recall that for UCB-Bernstein, an application of Freedman's inequality allows us to bound the martingale difference sequence in eq. \ref{eq:mdsbound} by $O\left(\sqrt{\frac{H}{t}[\mathbb{V}_hV^*_{h+1}]\iota + \frac{H^2}{t}\rmax\iota}\right)$ \citep{freedman1975tail}. \citet{jin2018q} show that in the bounded $[0,1]$ reward setting, this term is bounded by 
\begin{equation}
    \Theta\left(\sqrt{\frac{H}{t}(W_t(s,a,h)+H)\iota}+\frac{\iota\sqrt{H^7SA}}{t}\right)
\end{equation}

In the following, we modify the computation of this bound using eq. \ref{eq:vargap}:
\begin{equation}
\label{eq:new_vargap}
\begin{split}
    &\Theta\left(\sqrt{\frac{H\iota}{t}\left(W_t(s,a,h)\rmax + H\rmax + H^2\rmax^2 \sqrt{\frac{\iota}{t}} +  H^3u^\oeps\iota^\eeps\left(\frac{SA}{t} + \left(\frac{SA\iota}{t}\right)^\eeps\right)\right)} + \frac{\iota\sqrt{H^7SA\rmax}}{t}\right)\\
    \leq \,&\Theta\left(\sqrt{\frac{H\iota}{t}\left(W_t(s,a,h)\rmax + H\rmax + \frac{1}{1+\epsilon}\left(H+ \epsilon\frac{H^{\frac{2+3\epsilon}{\epsilon}}SA\iota}{t}\right)\right)} + \frac{H^2\iota\sqrt{\rmax^3}}{t} + \frac{\iota\sqrt{H^7SA\rmax}}{t}\right)\\
    \leq \,&\Theta\left(\sqrt{\frac{H\rmax\iota}{t}\left(W_t(s,a,h) + H\right)} + \frac{H^{\frac{1+2\epsilon}{\epsilon}}\iota\sqrt{SA\epsilon}}{t} + \frac{H^2\iota\sqrt{\rmax^3}}{t} + \frac{\iota\sqrt{H^7SA\rmax}}{t}\right) \leq \beta_t
\end{split}
\end{equation}
where we repeatedly leverage the (weighted) AM-GM inequality.

\paragraph{Analogue of Lemma C.5.} We now show that the analysis of the total variance introduces an unavoidable $\rmax^2$ factor. The analysis follows identically as that of Lemma C.5 in \citet{jin2018q}, except noting that $V$ is now bounded by $H\rmax$, which implies that the total variance is bounded by $H^2\rmax^2$. Namely, we have the following:
\begin{equation}
\label{eq:varpibound}
    \sum_{k=1}^K \sum_{h=1}^H \mathbb{V}_hV^{\pi_k}_{h+1} \leq O\left(\rmax^2\left(HT+H^3\iota\right)\right)
\end{equation}

\paragraph{Analogue of Lemma C.6.} We also state the following analogue of Lemma C.6 in \citet{jin2018q}, upper bounding $W_t$ via eq. \ref{eq:varpibound}. The proof of this lemma is a straightforward modification of that for the bounded reward setting.
\begin{lemma}
\label{lem:wbound}

There exists constant $c$ such that the following holds with probability $1-4\delta$
\begin{equation*}
\begin{split}
W_t(s,a,h) \leq \,&\mathbb{V}_hV^{\pi^k_h}_{h+1}(s,a) + 2H\rmax(\delta^k_{h+1}+\xi^k_{h+1})\\
&+ c\left(H^2\rmax^2\sqrt{\frac{\iota}{t}} + \rmax\sqrt{H^7\iota}\left(\frac{SA}{t}+\sqrt{\frac{SA}{t}}\right) + H^3u^\oeps\left(\frac{\iota^\eeps SA}{t} + \frac{(SA\iota)^\eeps}{t^\eeps}\right)\right)
\end{split}
\end{equation*}
\end{lemma}

\paragraph{Computing total regret.}
Here, we show the necessary modifications to the computation of regret akin to those in Theorem 2 \citet{jin2018q}. We first re-state the recursive expansion of the regret:

\begin{equation}
\label{eq:reg_bound}
    \sum_{k=1}^K \delta^k_h \leq SAH^2 + \sum_{h'=h}^H \sum_{k=1}^K (\beta_{n^k_{h'}}(s^k_{h'}a^k_{h'}, h') + \xi^k_{h+1})
\end{equation}
where $\xi^k_{h+1} \defeq [(\hat{\mathbb{P}}^{k_i}_h - \mathbb{P}_h)(V^*_{h+1}-V^k_{h+1})](s,a)$. By Azuma-Hoeffding, we have with probability $1-\delta$
\begin{equation}
\label{eq:xi_bound}
    \sum_{h'=h}^H\sum_{k=1}^K \xi^k_{h'} \leq \Theta(H\rmax\sqrt{T\iota})
\end{equation}
 
We also recall eq. C.13 from \citet{jin2018q}, scaled by a $\rmax$ factor in our setting:
\begin{equation}
\label{eq:loose_delta_bound}
\sum_{k=1}^K \delta^k_h \leq \Theta(\rmax\sqrt{H^4SAT\iota})
\end{equation}

We now proceed to bound $\sum_{k=1}^K\sum_{h=1}^H \beta_{n^k_h}$:
\begin{equation}
\label{eq:beta_bound}
\begin{split}
    \sum_{k=1}^K\sum_{h=1}^H \beta_{n^k_h} \leq &\sum_{k=1}^K\sum_{h=1}^H \Theta\left(\sqrt{\frac{H\rmax}{t}(W_t(s,a,h) + H)\iota} + \frac{\iota\sqrt{H^7\rmax SA}}{t} + \frac{H^2\iota\sqrt{\rmax^3}}{t} + \frac{H^{\frac{1+2\epsilon}{\epsilon}}\iota\sqrt{SA\epsilon}}{t}\right) + \\
    &\sum_{k=1}^K\sum_{h=1}^H H u^{\oeps}(\iota/t)^{\eeps}
\end{split}
\end{equation}

The second sum is bounded similarly as in the analysis of UCB-Hoeffding:
\begin{equation}
    \sum_{k=1}^K\sum_{h=1}^H H u^{\oeps}(\iota/t)^{\eeps} \leq \Theta\left(H^2(SA\iota)^{\eeps}T^{\oeps}\right)
\end{equation}

We now bound the second, third, and fourth terms in the first sum using the fact that $\sum_{k=1}^K\sum_{h=1}^H \frac{1}{n^k_h} \leq HSA\iota$:
\begin{equation}
\begin{split}
    \sum_{k=1}^K\sum_{h=1}^H \frac{\iota\sqrt{H^7\rmax SA}}{t} &\leq \sqrt{H^9\rmax S^3A^3\iota^4}\\
    \sum_{k=1}^K\sum_{h=1}^H\frac{H^2\iota\sqrt{\rmax^3}}{t} &\leq H^3SA\iota^2\sqrt{\rmax^3}\\
    \sum_{k=1}^K\sum_{h=1}^H \frac{H^{\frac{1+2\epsilon}{\epsilon}}\iota\sqrt{SA\epsilon}}{t} &\leq H^{\frac{1+3\epsilon}{\epsilon}}\sqrt{S^3A^3\iota^4\epsilon}\\
\end{split}
\end{equation}

We now bound $\sum_{k=1}^K\sum_{h=1}^H \sqrt{\frac{H\rmax}{t}(W_t(s,a,h) + H)\iota}$. From eq. C.15 in \citet{jin2018q}, this is upper bounded by $\sqrt{\sum_{k=1}^K\sum_{h=1}^H W_t(s,a,h)} \cdot \sqrt{H^2\rmax SA\iota} + \sqrt{H^3\rmax SAT\iota}$.

We now have
\begin{equation}
\begin{split}
    \skh W_{n^k_h}(s,a,h) \leq &\skh \left(\mathbb{V}_hV^{\pi^k_h}_{h+1}(s^k_h, a^k_h)+2H\rmax(\delta^k_{h+1}+\xi^k_{h+1}\right) +\\
    &\skh \Theta\left(H^2\rmax^2\sqrt{\frac{\iota}{t}} + \rmax\sqrt{H^7\iota}\left(\frac{SA}{t}+\sqrt{\frac{SA}{t}}\right) + H^3u^\oeps\left(\frac{\iota^\eeps SA}{t} + \frac{(SA\iota)^\eeps}{t^\eeps}\right)\right)
\end{split}
\end{equation}

Substituting equations \ref{eq:varpibound}, \ref{eq:xi_bound}, and \ref{eq:loose_delta_bound}, this is bounded by
\begin{equation}
\begin{split}
    \Theta\Bigl(&\rmax^2\sqrt{H^8SAT\iota} + H\rmax^2 T + H^3\rmax^2\iota + \rmax S^2A^2\sqrt{H^9\iota^3} +\\
    &\rmax SA\sqrt{H^8T\iota} + H^3u^\oeps\iota^{\frac{1+2\epsilon}{1+\epsilon}} S^2A^2 + H^{\frac{2+3\epsilon}{1+\epsilon}}u^\oeps(S^2A^2\iota)^\eeps T^\oeps\Bigr)
\end{split}
\end{equation}

Noting that $\epsilon^2\leq 1+\epsilon$ and $\rmax \leq u^\oeps$ by Cauchy-Schwarz, we can simplify this bound into the following:
\begin{equation}
   \skh W_{n^k_h}(s,a,h) \leq \Theta \left( H\rmax^2 T + \rmax^2S^2A^2H^7\iota + H^{\frac{1+2\epsilon}{1+\epsilon}}SA\iota + u^\oeps S^2A^2\sqrt{H^9\iota^3}\right)
\end{equation}
    
Hence, we have the following:
\begin{equation}
\begin{split}
    \sqrt{\sum_{k=1}^K\sum_{h=1}^H W_{n^k_h}(s,a,h)} \cdot \sqrt{H^2\rmax SA\iota} &\leq \Theta\left(\sqrt{H^3\rmax^3 SAT\iota} + \sqrt{H^{\frac{1+4\epsilon}{\epsilon}}\rmax S^2A^2\iota^2} + \sqrt{H^9\rmax^2u^\oeps S^3A^3\iota^3}\right)
\end{split}
\end{equation}

Substituting everything back into eqs. \ref{eq:reg_bound} and \ref{eq:beta_bound} yields the desired result.

\end{proof}

\section{Experiment Details}

\subsection{Synthetic MDPs}

Here, we describe the SixArms and DoubleChain MDPs in further detail. SixArms consists of six states $s_1, \ldots, s_6$ rearranged around a central initial state $s_0$ \citep{strehl2008analysis}. There are six actions, and only taking action $i-1$ in state $i$ for $i=1, \ldots, 6$ yields positive reward. Taking action $i$ in state $s_0$ transitions to state $i$ with probability $p_i$, and taking action $i-1$ state $s_i$, $i>0$ transitions to itself with probability $1$, and taking any action transitions back to $s_0$. Letting $r_i$ denote $r(i, i-1)$, we set $r_1$ to $\mathcal{N}(1.20, 0.1^2)$, and $r_i$ to $\mathcal{L}(\mu_i, 1.1, 0, 1)$ where $\mu_i=1+0.2(i-1)$ for $i=1,\ldots, 6$. On SixArms, we run all algorithms for $5 \cdot 10^5$ episodes of length $25$ for $30$ random seeds. 

DoubleChain, pictured in figure \ref{fig:double_chain}, begins at an initial state $s_0$, and can transition to either of two RiverSwim-style MDPs of length $l$ \citep{dimakopoulou2018coordinated}. As in RiverSwim, only the rightmost states have significant positive reward; all other states have either zero or slightly negative mean reward. In our DoubleChain experiments, the reward distributions for the two rightmost states $s_l$ and $s_{2l+1}$ are set to $\mathcal{N}(0.5, 0.1^2)$ and symmetric ($\beta=0$) Levy $\alpha$-stable with $\alpha=1.1$, $\mu=1$ respectively.  The rewards for states $1,\ldots, l-1$ are set to $\mathcal{N}(0, 0.1^2)$, and states $l+1,\ldots, 2l$ to $\mathcal{N}(-0.1, 0.01^2)$. We set $p=0.8$, $l=3$, and run all algorithms for $5\cdot10^{5}$ episodes of length $10$ for $30$ random seeds. 

\subsection{Deep RL experiments}

We used DQN with the following hyperparameters:

\begin{table}[h!]
\centering
\begin{tabular}{|l|l|}
\hline
Hyperparameter                   & Value                                                \\ \hline
$n$-step update horizon          & 1                                                    \\ 
Discount ($\gamma$) & 0.99                                                 \\ 
Adam epsilon                     & 0.0003125                                            \\ 
Minibatch size                   & 64                                                   \\ 
Max replay buffer size           & 500000                                               \\ 
Min replay buffer size           & 500                                                  \\   
Learning rate & 1e-4 linearly annealed to 1e-5                       \\ 
Epsilon greedy schedule          & 1 linearly annealed to 0.01 over first 10\% of steps \\ 
$Q$-network update period        & 2                                                    \\ 
Target $Q$-network update period & 50                                                   \\ \hline
\end{tabular}
\vspace{1ex}
\caption{Shared DQN hyperparameters}
\end{table}

\end{document}